\documentclass{article}

\newif\iflongversion
\newif\ificml
\icmlfalse
\ificml
\longversionfalse
\usepackage{icml2019}
\else
\longversiontrue
\usepackage{jmlr2e}
\fi

\usepackage{natbib}
\usepackage{chngcntr} 
\usepackage{microtype}
\usepackage{graphicx}
\usepackage{subfigure}
\usepackage{multirow}
\usepackage{booktabs}
\usepackage[utf8]{inputenc}
\usepackage[french,english]{babel}
\usepackage{refcount}
\usepackage{amsmath}
\usepackage{amsfonts}
\ificml
\usepackage{amsthm}
\else
\fi
\usepackage{bbm}
\usepackage{cuted}
\usepackage[]{algorithm2e}
\usepackage{esvect}
\usepackage{float}

\ificml
\else
\usepackage[
left=2.6cm,
right=2.6cm,
top=2.6cm,
bottom=2.6cm
]{geometry}
\fi

\ificml
\usepackage{hyperref}
\else
\fi

\newif\ifoldloss
\oldlossfalse

\newif\ifplanB
\planBtrue

\DeclareMathOperator{\Categorical}{Categorical}
\DeclareMathOperator{\Dirichlet}{Dirichlet}
\DeclareMathOperator{\softmax}{softmax}

\newcommand{\regret}{\mathcal{R}}
\ificml
\newcommand{\sigm}{\sigma}
\else
\newcommand{\sigm}{{\rm sigmoid}}
\fi
\newcommand{\1}{{\mathbbm{1}}}

\newtheorem{prop}{Proposition}

\ificml
\icmltitlerunning{Meta-Transfer Objective for Learning to Disentangle Causal Mechanisms}
\else
\fi

\begin{document}

\ificml
\else
\title{A Meta-Transfer Objective for Learning to Disentangle Causal Mechanisms}
\fi

\ificml
\else
\author{
Yoshua Bengio$^{1,2,5}$, Tristan Deleu$^1$, Nasim Rahaman$^4$, Nan Rosemary Ke$^3$, Sébastien Lachapelle$^1$,\\
Olexa Bilaniuk$^1$, Anirudh Goyal $^1$  and  Christopher Pal$^{3,5}$\\
       Mila,
       Montréal, Québec, Canada\\
       $^1$ Université de Montréal\\
       $^2$ CIFAR Senior Fellow\\
       $^3$ École Polytechnique Montréal\\
       $^4$ Ruprecht-Karls-Universit\"at Heidelberg \\
       $^5$ Canada CIFAR AI Chair
      }

\date{}

\editor{}
\maketitle
\fi

\begin{abstract}
We propose to meta-learn causal structures based on how fast a learner adapts to new distributions arising from sparse distributional changes, e.g. due to interventions, actions of agents and other sources of non-stationarities. We show that under this assumption, the correct causal structural choices lead to faster adaptation to modified distributions because the changes are concentrated in one or just a few mechanisms when the learned knowledge is modularized appropriately. This leads to sparse expected gradients and a lower effective number of degrees of freedom needing to be relearned while adapting to the change. It motivates using the speed of adaptation to a modified distribution as a meta-learning objective. We demonstrate how this can be used to determine the cause-effect relationship between two observed variables.  The distributional changes do not need to correspond to standard interventions (clamping a variable), and the learner has no direct knowledge of these interventions. We show that causal structures can be parameterized via continuous variables and learned end-to-end. We then explore how these ideas could be used to also learn an encoder that would map low-level observed variables to unobserved causal variables leading to faster adaptation out-of-distribution, learning a representation space where one can satisfy the assumptions of independent mechanisms and of small and sparse changes in these mechanisms due to actions and non-stationarities.
\end{abstract}

\ificml
\else

\vspace*{6mm}
\fi

\section{Introduction}
Current machine learning methods seem weak when they are required to generalize beyond the training distribution, which is what is often needed in practice. It is not enough to obtain good generalization on a test set sampled from the same distribution as the training data, we would also like what has been learned in one setting to generalize well in other related distributions. These distributions may involve the same concepts that were seen previously by the learner, with the changes typically arising because of actions of agents. More generally, we would like what has been learned previously to form a rich base from which very fast adaptation to a new but related distribution can take place, i.e., obtain good transfer.  
\iflongversion
Some new concept may have to be learned but because most of the other relevant concepts have already been captured by the learner (as well as how they can be composed), learning can be very fast on the transfer distribution. 
\fi

Short of any assumption, it is impossible to have a successful transfer to an unrelated distribution. In this paper we focus on the assumption that the changes are sparse when the knowledge is represented in an appropriately modularized way, with only one or a few of the modules having changed. This is especially relevant when the distributional change is due to actions by one or more agents, such as the interventions discussed in the causality literature~\citep{pearl2009causality,peters2017elements}, where a single causal variable is clamped to a particular value. In general, it is difficult for agents to influence many underlying causal variables at a time, and although this paper is not about agent learning as such, this is a property of the world that we propose to exploit here, to help discovering these variables and how they are causally related to each other. 

To motivate the need for inferring causal structure, consider that interventions may be actually performed or may be imagined. In order to properly plan in a way that takes into account interventions, one needs to imagine a possible change to the joint distribution of the variables of interest due to an intervention, even one that has never been observed before. This goes beyond good transfer learning and requires causal learning and causal reasoning. For this purpose, it is not sufficient to learn the joint distribution of the observed variables. One also should learn enough about the underlying high-level variables and their causal relations to be able to properly infer the effect of an intervention. For example, $A$={\bf Raining} causes $B$={\bf Open Umbrella} (and not vice-versa). Changing the marginal probability of {\bf Raining} (say because the weather changed) does not change the mechanism that relates $A$ and $B$ (captured by $P(B|A)$), but will have an impact on the marginal $P(B)$. Conversely, an agent's intervention on $B$ ({\bf Open umbrella}) will have no effect on the marginal distribution of $A$ ({\bf Raining}). That asymmetry is generally not visible from the $(A,B)$ training pairs alone, until a change of distribution occurs, e.g. due to an intervention.  This motivates the setup of this paper, where one learns from a set of distributions arising from not necessarily known interventions, not simply to capture a joint distribution but to discover the some underlying causal structure.

\iflongversion 
Machine learning methods are often exploiting some form of assumption about the data distribution (or else, the no free lunch theorem tells us that we cannot have any confidence in generalization). 
\fi
In this paper, we are considering not just assumptions on the data distribution but also on how it changes (e.g., when going from a training distribution to a transfer distribution, possibly resulting from some agent's actions). We propose to rely on the assumption that, {\em when the knowledge about the distribution is appropriately represented, these changes would be small}. This arises because of an {\bf underlying assumption} (but more difficult to verify directly) {\bf that only one or few of the ground truth mechanisms have been changed}, due to some generalized form of intervention leading to the modified distribution.

How can we exploit this assumption? As we explain theoretically and verify experimentally here, if we have the right knowledge representation, then we should get fast adaptation to the transfer distribution when starting from a model that is well trained on the training distribution. This arises because of our assumption that the ground truth data generative process is obtained as the composition of independent mechanisms and that, very few ground truth mechanisms and parameters need to change when going from the training distribution to the transfer distribution. A model capturing a corresponding factorization of knowledge would thus require just a few updates, a few examples, for this adaptation to the transfer distribution. As shown below, the expected gradient on the unchanged parameters would be near 0 (if the model was already well trained on the training distribution), so the effective search space during adaptation to the transfer distribution would be greatly reduced, which tends to produce fast adaptation, as found experimentally.

Thus, based on the assumption of small change in the right knowledge representation space, we can define a meta-learning objective that measures the speed of adaptation, i.e., a form of regret, in order to optimize the way in which knowledge should be represented, factorized and structured. {\em This is the core idea presented in this paper}. Note that a stronger signal can be obtained when there are more non-stationarities, i.e., many changes in distribution, just like in meta-learning we get better results with more meta-examples.

In this way, we can take what is normally considered a nuisance in machine learning (changes in distribution due to non-stationarity, uncontrolled interventions, etc.) and turn that into a training signal to find a good way to factorize knowledge into components and mechanisms that match the assumption of small change. Thus, we end up optimizing in an end-to-end way the very thing we care about at the end, i.e. fast transfer and robustness to distributional changes. If the data was really generated from the composition of independent causal mechanisms~\citep{peters2017elements}, then there exists a good factorization of knowledge that mimics that structure. If in addition, at each time step, agents in the real world tend to only be able to change one or very few high-level variables (or the associated mechanisms producing them), then our assumption of small change (in the right representation) should be generally valid. Also, in addition to obtaining fast transfer, we may be able to recover a good approximation of the true causal decomposition into independent mechanisms (to the extent that the observations and interventions can reveal those mechanisms).  

In this paper, we begin exploring the above ideas with specific experiments on synthetically generated data in order to validate them and demonstrate the existence of simple algorithms to exploit them. However it is clear to us that much more work will be needed to evaluate the proposed approach in a diversity of settings and with different specific parametrizations, training objectives, environments, etc. We begin with what are maybe the simplest possible settings and evaluate whether the above approach can be used to learn the direction of  
\ifplanB
causality.
\fi
We then study the crucial question of obtaining a training signal about how to transform raw observed data into a representation space where the latent variables can be modeled by a sparse causal graph with sparse distributional changes and show results that confirm that the correct encoder leads to a better value of our expected regret meta-learning objective.

\section{Which is Cause and Which is Effect?}
\label{sec:two-observed-variables}

To anchor ideas and show an example of application of the above-proposed meta-objective for knowledge decomposition, we consider in this section the problem of determining if variable $A$ causes variable $B$ or vice-versa. The learner observes training samples $(a,b)$ from a pair of related distributions, which by convention we call the training distribution and the transfer distribution. Note that based only on samples from a single (training) distribution, in general both the $A \rightarrow B$ model ($A$ causes $B$) and the $B \rightarrow A$ model (vice-versa, see Equation~\eqref{eq:bivariate-models} below) tend to perform as well in terms of ordinary generalization (to a test set sampled from the training distribution), see also a theoretical argument and simulation results in Appendix~\ref{sec:non-identifiability}. To highlight the power of the proposed meta-learning objective, we consider the situation where lots of examples are available for the training distribution but very few for the transfer distribution. In fact, as we will argue below, the training signal that will allow us to infer the correct causal direction will be stronger if we have access to many short transfer adaptation episodes. Short episodes are most informative because after having seen a lot of data from the transfer distribution, it will not matter much whether $A$ causes $B$ or vice-versa (when there is enough training data compared to the number of free parameters, both models converge towards an optimal estimation of the joint). However, in order to generalize quickly from very few examples of the transfer distribution, it does matter to have made the correct choice of the causal direction. Let us now justify this in more detail below and then demonstrate this by simulations.

\subsection{Learning a Causal Graph with two Discrete Variables}

Let both $A$ and $B$ be discrete variables each taking $N$ possible values and consider the following two parametrizations (the $A\rightarrow B$ model and the $B \rightarrow A$ model) to estimate their joint distribution:
\begin{align}
\label{eq:bivariate-models}
    P_{A\rightarrow B}(A,B)&=P_{A\rightarrow B}(A) P_{A\rightarrow B}(B\mid A) \nonumber \\
    P_{B\rightarrow A}(A,B)&=P_{B\rightarrow A}(B) P_{B\rightarrow A}(A \mid B) 
\end{align}
Each of these two graphical models (denoted $A \rightarrow B$ and $B \rightarrow A$) decomposes the joint into two separately parametrized modules, each corresponding to a different causal mechanism associated with the probability of a variable given its parents in the graph. This amounts to four modules: $P_{A\rightarrow B}(A)$, $P_{A\rightarrow B}(B\mid A)$, $P_{B\rightarrow A}(B)$ and $P_{B\rightarrow A}(A\mid B)$. We will train both models independently. Since we assume in this section that the pairs $(A,B)$ are completely observed, we can use a simple maximum likelihood estimator to independently train all four 
\iflongversion modules (the log-likelihood of the joint decomposes into separate objective functions, one for each conditional, in a directed graphical model with fully observed variables).
\else
modules.
\fi
In the discrete case with tabular parametrization, the maximum likelihood estimator can be computed analytically, and corresponds to the appropriately normalized relative frequencies. Let $\theta$ denote the parameters of all these models, split into sub-vectors for each module, e.g., $\theta_{A|B}$ for the $N^2$ conditional probabilities for each possible value of $B$ and each possible value of $A$. In our experiments, we parametrized these probabilities via softmax of unnormalized quantities.

\subsubsection{The Advantage of the Correct Causal Model}

First, let us consider simply the likelihood of the training data only (i.e., no change of distribution) for the different causal models considered. Both models have $O(N^2)$ parameters, and maximum likelihood estimation leads to indistinguishable test set performance (where the test set is sampled from the training distribution). See Appendix~\ref{sec:non-identifiability} for a demonstration that both models would have the same likelihood, and associated experimental results. These results are not surprising in light of the existing literature on non-identifiability of causality from observations~\citep{pearl2009causality,peters2017elements}, but they highlight the importance of using changes in distribution to provide a signal about the causal structure.

Now instead let us compare the performance of our two hypotheses ($A\rightarrow B$ vs $B \rightarrow A$) in terms of how fast the two models adapt on a transfer distribution after having been trained on the training distribution.  We will assume
simple stochastic gradient descent on the parameters for this
adaptation but other procedures could be used, of course. Without loss of generality, let $A \rightarrow B$ be the correct causal model. To make the case stronger, let us consider that the change between the two distributions amounts to a random change in the parameters of the true $P(A)$ for the cause $A$ (because this will have an impact on the effect $B$, which can be picked up and reveal the causal direction). We do not assume that the learner knows what intervention was performed, unlike in more common approaches to causal discovery and controlled experiments. We only assume that some change happened and we try to exploit that to reveal structural causal information. 

\begin{figure}[H]
    \centering
    \ificml
\else
    \includegraphics[width=0.6\linewidth]{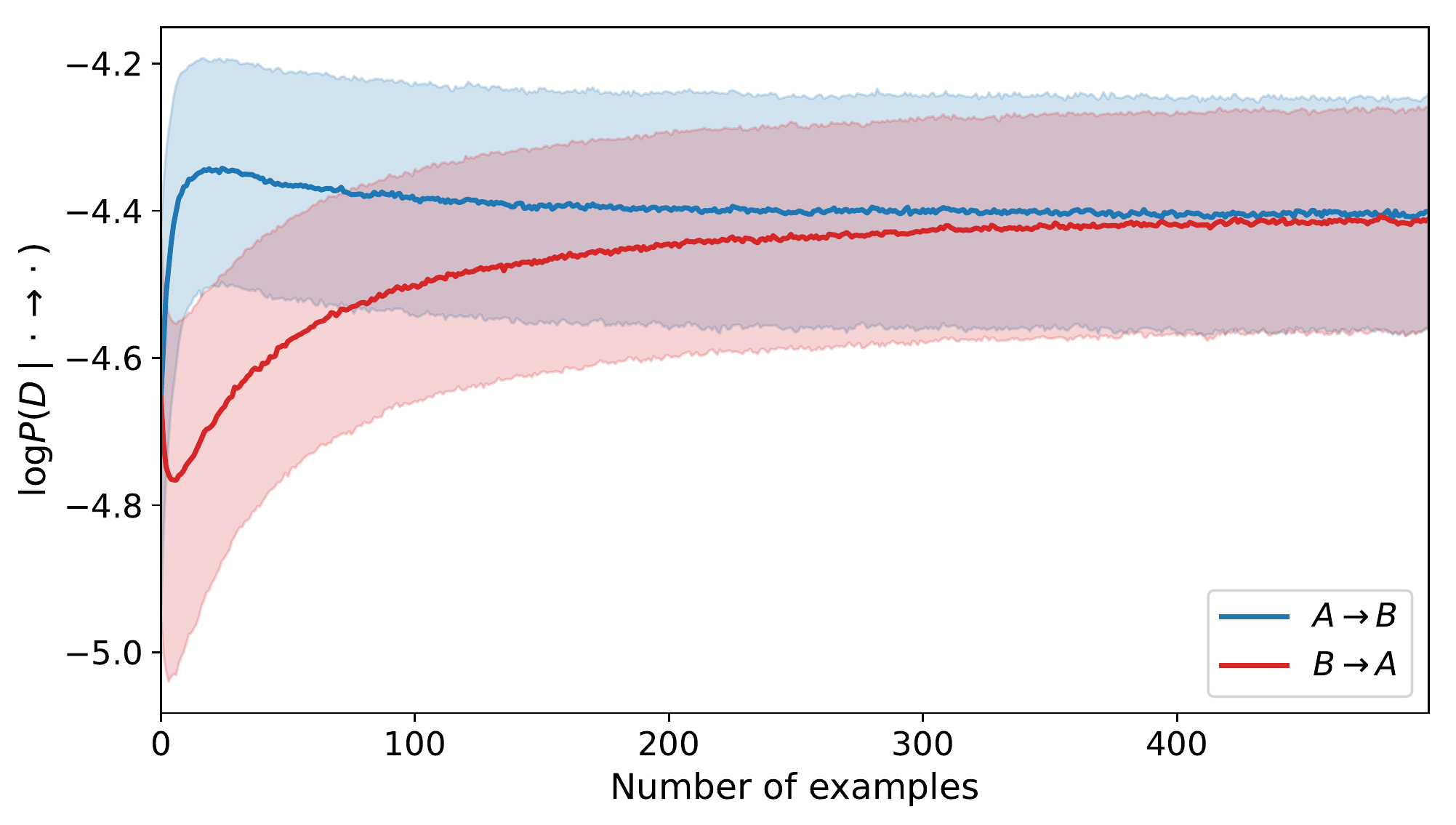}
    \fi
    \caption{Adaptation to the transfer distribution, as more transfer distribution examples are seen by the learner (horizontal axis), in terms of the log-likelihood on the transfer distribution (on a large test set from the transfer distribution, tested after each update of the parameters). Here the model is discrete, with $N = 10$. Curves are the median over 10\,000 runs, with 25-75\% quantiles intervals, for both the correct causal model (blue, top) and the incorrect one (red, bottom). We see that the correct causal model adapts faster (smaller regret), and that the most informative part of the trajectory (where the two models generalize the most differently) is in the first 10-20 examples.}
    \label{fig:adaptation-curve}
\end{figure}

\ifplanB
\subsection{Experiments on Adaptation to the transfer distribution}
\label{app:adaptation-transfer}

\fi

\ifplanB
We present experiments comparing the learning curve of the correct causal model on the transfer distribution vs the learning curve of the incorrect model.
The adaptation with only a few gradient steps on data coming from a different, but related, transfer distribution is critical in getting a signal that can be leveraged by our meta-learning algorithm. To show the effect of this adaptation, and motivate our use of only a small amount of data from the transfer distribution, we experimented with a model on discrete random variables taking $N = 10$ possible values.

In this experiment, we fixed the underlying causal model to be $A \rightarrow B$, and trained the modules for each marginal and conditional distributions with maximum likelihood on a large amount of data from some training distribution, as explained in Appendix~\ref{sec:non-identifiability}. See also Appendix~\ref{sec:bivariate-experiment-discrete} and Table~\ref{tab:categorical-modules} for details on the definitions of these modules.

We then adapt all the modules on data coming from a transfer distribution, corresponding on an intervention on the random variable $A$ (i.e., the marginal $P(A)$ of the ground truth model is modified, while leaving $P(B\mid A)$ fixed). We used RMSprop for the adaptation, with the same learning rate. For assessing reproducibility and statistical robustness, the experiment was repeated over 100 different training distributions, and over 100 transfer distributions for each training distributions, leading to 10\,000 experiments overall. The procedure to acquire different training/transfer distributions is detailed in Appendix~\ref{sec:bivariate-experiment-discrete}.

In Figure \ref{fig:adaptation-curve}, we report the log-likelihoods of both models, evaluated on a large test set of 10\,000 from the transfer distribution. We can see that as the number of examples from the transfer distribution (equal to the number of adaptation steps) increases, the two models eventually reach the same log-likelihood, reflecting our observation from Appendix~\ref{sec:non-identifiability}. However the causal model $A \rightarrow B$ adapts faster than the other model $B \rightarrow A$, with the most informative part of the trajectory (where the difference is the largest) is within the first 10 to 20 examples.
\fi

\subsubsection{Parameter Counting Argument}
\label{sec:parameter-counting}

A simple parameter counting arguments helps us understand what we are observing in Figure~\ref{fig:adaptation-curve}. 
First, consider the expected gradient on the parameters of the different modules, during the adaptation phase to the transfer distribution, which we designate as adaptation episode, and corresponds to learning from a meta-example.
\begin{prop}
\label{prop:zero-gradient}
The expected gradient over the transfer distribution of the regret (accumulated negative log-likelihood during the adaptation episode) with respect to the module parameters is zero for the parameters of the modules that (a) were correctly learned in the training phase, and (b) have the correct set of causal parents, corresponding to the ground truth causal graph, if (c) the corresponding ground truth conditional distributions did not change from the training distribution to the transfer distribution. 
\end{prop}
The proof is given in Appendix~\ref{sec:proof-zero-gradient}. The basic justification for this proposition is that for the modules that were correctly learned in the training distribution and whose ground truth conditional distribution did not change with the transfer distribution, the parameters already are at a maximum of the log-likelihood over the transfer distribution, so the expected gradient is zero.

As a consequence, the effective number of parameters that need to be adapted, when one has the correct causal graph structure, is reduced to those of the mechanisms that actually changed from the training to the transfer distribution. Since sample complexity - the number of training examples necessary to learn a model - grows approximately linearly~\citep{ehrenfeucht1989} with VC-dimension~\citep{Vapnik71}, and since VC-dimension grows approximately linearly in the number of parameters in linear models and neural networks~\cite{shalev-shwartz-book2014}, the learning curve on the transfer distribution will tend to improve faster for the model with the correct causal structure, for which fewer parameters need to be changed. Interestingly, we do not need to have the whole causal graph correctly specified before getting benefits from this phenomenon. If we only have part of the causal graph correctly specified and we change our causal hypothesis to include one more correctly specified mechanism, then we will obtain a gain in terms of the adaptation sample complexity (which shows up when the change in distribution does not touch that mechanism). 
This nice property also shows up in Proposition~
\ifoldloss
\ref{prop:unbiased} 
\else
\ref{prop:biased}
\fi
(Appendix~\ref{sec:many-factors}), showing a decoupling of the meta-objective across the independent mechanisms.

Let us consider the special case we have been studying up to now. We have four modules, two of which ($P_{A\rightarrow B}(A)$ and $P_{B\rightarrow A}(B)$) are marginal discrete distributions over $N$ values, which require each $N-1$ free parameters. The other two modules are conditional probability tables that have $N$ rows each with $N-1$ free parameters, i.e., a total of $N(N-1)$ free parameters. If $A\rightarrow B$ is the correct model and the transfer distribution only changed the true $P(A)$ (the cause), and if $P(B\mid A)$ had been correctly estimated on the training distribution, then for the correct model only $N-1$ parameters need to be re-estimated. On the other hand, because of Bayes' rule, under the incorrect model ($B \rightarrow A$), a change in $P(A)$ leads to new parameters for both $P(B)$ and $P(A\mid B)$, i.e., all $N(N-1)+(N-1)=N^2-1$ parameters must be re-estimated. In this case we see that sample complexity may be $O(N^2)$ for the incorrect model while it would be $O(N)$ for the correct model (assuming linear relationship between sample complexity and number of free parameters). Of course, if the change in distribution had been over $P(B\mid A)$ instead of $P(A)$, the advantage would not have been as great. This would motivate information gathering actions generally resulting in a very sparse change in the mechanisms.

\subsection{Smooth parameterization of the causal structure}

In the more general case with many more than two hypotheses for the structure of the causal graph, there will be an exponentially large set of possible causal structures explaining the data and we won't be able to enumerate all of them (and pick the best one after observing episodes of adaptation). However, we can parameterize our belief about an exponentially large set of hypotheses by keeping track of the probability for each directed edge of the graph to be present, i.e., specify for each variable $B$ whether some variable $A$ is a direct causal parent of $B$ (for all pairs $(A,B)$ in the graph). 
\ifplanB
We will develop such a smooth parametrization further in Appendix~\ref{sec:many-factors}, but it 
hinges on gradually changing our belief in the individual binary decisions associated with each edge of the causal graph, so we can jointly do gradient descent on all these beliefs at the same time.

In this section, we study the simplest possible version of this idea, representing that edge belief via a structural parameter $\gamma$ with $\sigm(\gamma)={\rm sigmoid}(\gamma)$, our believed probability that $A\rightarrow B$ is the correct choice. For that single pair of variables scenario, let us consider two explanations for the data (as in the above sections, for models $A\rightarrow B$ and $B\rightarrow A$), one with probability $p(A\rightarrow B)=\sigm(\gamma)$ and the other with probability $p(B\rightarrow A)=1-\sigm(\gamma)$. We can write down our transfer objective as a log-likelihood over the mixture of these two models. Note this is different from the usual mixture models, which assume separately for each example that it was sampled from one component or another with some probability. Here, we assume that all of the observed data was sampled from one component or the other. 
The transfer data regret (negative log-likelihood accumulated along the online adaptation trajectory) under that mixture is therefore as follows:
\begin{equation}
\label{eq:mixture-ll}
 \regret = -\log \left[ \sigm(\gamma) \mathcal{L}_{A \rightarrow B} + (1-\sigm(\gamma)) \mathcal{L}_{B \rightarrow A} \right]
\end{equation}
where $\mathcal{L}_{A\rightarrow B}$ and $\mathcal{L}_{B\rightarrow A}$ are the online likelihoods of both models respectively on the transfer data. They are defined as
\begin{align*}
    \mathcal{L}_{A\rightarrow B} &= \prod_{t=1}^{T}P_{A \rightarrow B}(a_{t}, b_{t}\,;\,\theta_{t})\\
    \mathcal{L}_{B\rightarrow A} &= \prod_{t=1}^{T}P_{B\rightarrow A}(a_{t}, b_{t}\,;\,\theta_{t}),
\end{align*}
where $\{(a_t,b_t)\}_t$ is the set of transfer examples for a given episode and $\theta_t$ aggregates all the modules' parameters as of time step $t$ (since the parameters could be updated after each observation of an example $(a_t,b_t)$  from the transfer distribution). $P_{\rm model}(a,b;\theta)$ is the likelihood of example $(a,b)$ under some {\em model} that has parameters $\theta$. 

The quantity of interest here is $\frac{\partial \regret}{\partial \gamma}$, which is our training signal for updating $\gamma$. In the experiments below, after each episode involving $T$ transfer examples we update $\gamma$ by doing one step of gradient descent, to reduce the transfer negative log-likelihood or regret $\regret$. What we are proposing is a meta-learning framework in which the inner training loop updates the module parameters (separately) as examples are seen (from either distribution being currently observed), while the outer loop updates the structural parameters (here it is only the scalar $\gamma$) with respect to the transfer negative log-likelihood.

The gradient of the transfer log-likelihood with respect to the structural parameter $\gamma$ is pushing $\sigm(\gamma)$ towards the posterior probability that the correct model is $A\rightarrow B$
and $(1-\sigm(\gamma))$ towards the posterior probability
that the correct model is $B\rightarrow A$:
\begin{prop}
\label{prop:posterior}
The gradient of the negative log-likelihood of the transfer data in Equation \eqref{eq:mixture-ll} wrt. the structural parameter $\frac{\partial \regret}{\partial \gamma}$ is given by
\begin{equation}
    \frac{\partial \regret}{\partial \gamma} = \sigma(\gamma) - P(A \rightarrow B \mid D_{2}),
\end{equation}
where $D_{2}$ is the transfer data, and $P(A \rightarrow B\mid D_{2})$ is the posterior probability of the hypothesis $A \rightarrow B$ (when the alternative is $B \rightarrow A$). Furthermore, this can be equivalently written as
\begin{equation}
    \frac{\partial \regret}{\partial \gamma} = \sigma(\gamma) - \sigma(\gamma + \Delta),
\end{equation}
where $\Delta = \log \mathcal{L}_{A \rightarrow B} - \log \mathcal{L}_{B \rightarrow A}$ is the difference between the log-likelihoods of the two hypotheses on the transfer data $D_{2}$.
\end{prop}
The proof is given in Appendix~\ref{sec:proof-posterior}. Note how this posterior probability is basically measuring which hypothesis is better explaining the episode transfer data $D_2$ overall along the adaptation trajectory. $D_2$ is a meta-example for updating the structural parameters like $\gamma$. Larger $\Delta$ of one hypothesis over the other leads to moving meta-parameters faster towards the favoured hypothesis. This difference in online accumulated log-likelihoods $\Delta$ also relates to log-likelihood scores in score-based methods for structure learning of graphical models \citep{koller2009probabilistic}\footnote{One can see $\log \mathcal{L}_{A \rightarrow B}$ as a score attributed to graph $A \rightarrow B$, analogously for $\log \mathcal{L}_{B \rightarrow A}$. The gradient is then pushing toward the graph with the highest score.}.

To find where SGD converges, note that the actual posterior depends on the prior $\sigm(\gamma)$ and thus keeps changing after each gradient step. We are really doing SGD on the expected value of $\regret$ over transfer sets $D_2$. Equating the gradient of this expected value to zero to look for the stationary convergence point, we thus see $\sigm(\gamma)$ on both sides of the equation, and we obtain convergence when the new value of $\sigm(\gamma)$ is consistent with the old value, as clarified in this proposition.
\begin{prop}
\label{prop:convergence}
Stochastic gradient descent (with appropriately decreasing learning rate) on $E_{D_2}[\regret]$ with steps from $\frac{\partial \regret}{\partial \gamma}$ converges towards $\sigm(\gamma)=1$ if $E_{D_2}[\log \mathcal{L}_{A \rightarrow B}]>E_{D_2}[\log \mathcal{L}_{B \rightarrow A}]$, or $\sigma(\gamma)=0$ otherwise.
\end{prop}
The proof is given in Section~\ref{sec:proof-convergence} of the Appendix, and shows that optimizing $\gamma$ will end up picking the correct hypothesis, i.e., the one that has the smallest regret (or fastest convergence), measured as the accumulated log-likelihood as adaptation proceeds on the transfer distributions sampled from the distribution $D_2$, which we can think of like a distribution over tasks, in meta-learning. This analogy with meta-learning also appears in our gradient-based adaptation procedure, which is linked to existing methods like the first-order approximation of MAML \citep{FinnAL17}, and its related algorithms \citep{NicholReptile2018}.
Algorithm~\ref{alg:main} (Appendix~\ref{sec:pseudo-code}) illustrates the general pseudo-code for the proposed meta-learning
framework.

\subsubsection{Experimental Results}

To illustrate the convergence result from Proposition~\ref{prop:convergence}, we experiment with learning the structural parameter $\gamma$ in a bivariate model, with discrete random variables, each taking $N=10$ and $N = 100$ possible values. In this experiment, we assume that the underlying causal model (unknown to the algorithm) is fixed to $A \rightarrow B$, so that we want the structural parameter to eventually converge to $\sigma(\gamma) = 1$. The details of the experimental setup can be found in Appendix~\ref{sec:bivariate-experiment-discrete}.


\begin{figure}[ht]
    \centering
    \ificml
    \includegraphics[width=0.99\linewidth]{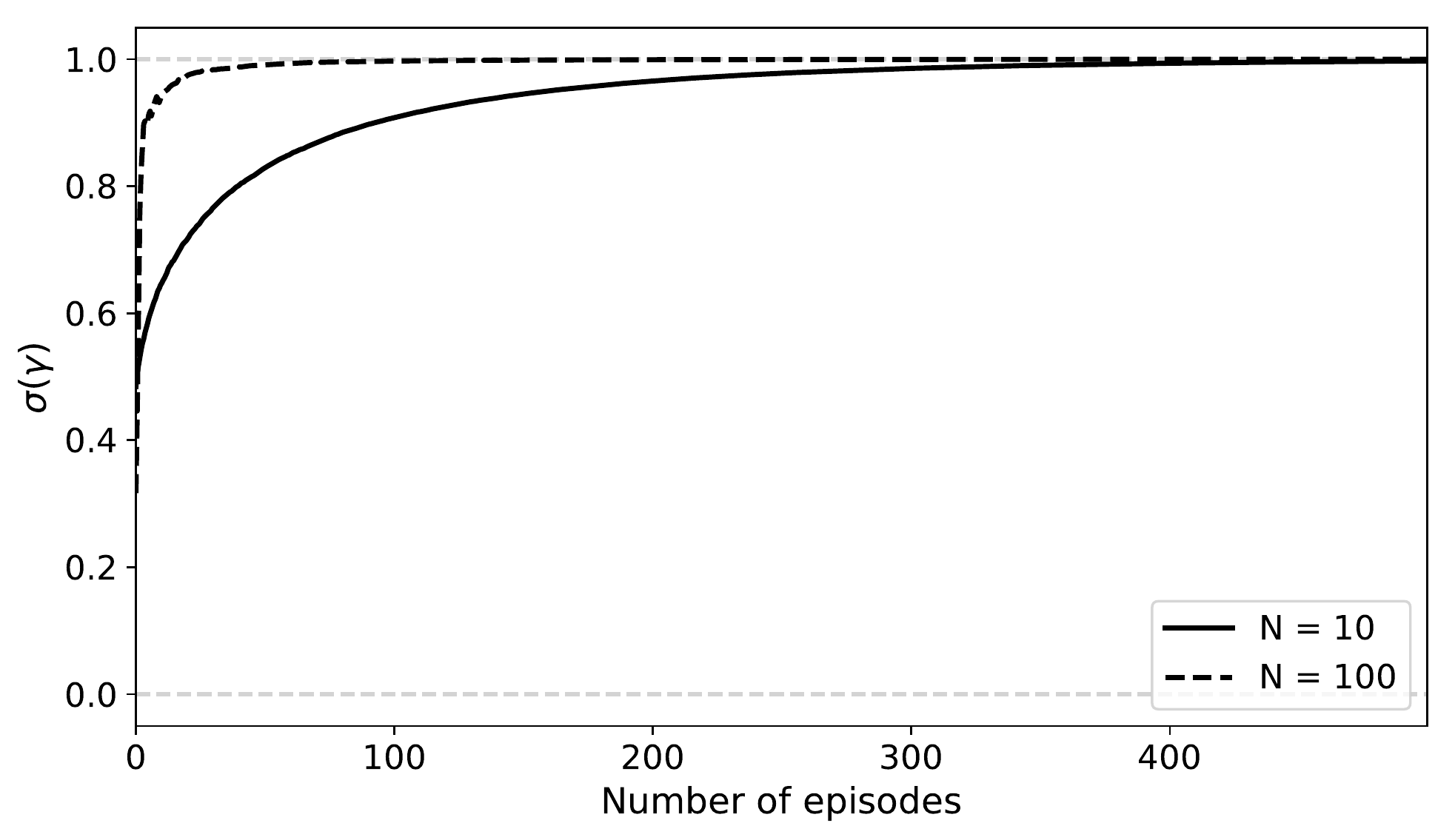}
    \else
    \includegraphics[width=0.6\linewidth]{alphas-metatrain-discrete-categorical.pdf}
    \fi
    \caption{Evolution of model's belief $p(A\rightarrow B) = \sigma(\gamma)$ on a bivariate model, with discrete random variables ($N = 10$ \& $N = 100$). The horizontal axis represents the number of episodes (i.e., meta-examples) seen during meta-training, which corresponds to the number of SGD updates of the structural parameter $\gamma$.}
    \label{fig:two-variables-alphas-discrete}
\end{figure}

In Figure~\ref{fig:two-variables-alphas-discrete}, we show the evolution of $\sigma(\gamma)$ (which is the model's belief of $A\rightarrow B$ being the correct causal model) as the number of episodes increases. Starting from an equal belief for both $A \rightarrow B$ and $B \rightarrow A$ to occur ($\sigma(\gamma) = 0.5$), the structural parameter converges to $\sigma(\gamma) = 1$ within $500$ episodes. 

This observation is consistent across a range of domains, including models with multimodal or multivariate continuous variables, and different parametrizations of the models. In Appendix~\ref{sec:MLP2}, we present results for two discrete variables but using MLPs to parametrize the conditional distributions, and where there are more causal hypotheses: we consider one binary choice for each directed edge in the graph, to decide
whether one variable is a direct causal parent or not. Figure~\ref{fig:mlp_scm} shows that the correct causal graph is quickly recovered. To estimate the gradient, we use
a generalization of the regret loss (introduced above, Equation~\eqref{eq:mixture-ll})
and its gradient,
described in Appendix~\ref{sec:many-factors}.

\begin{figure}[ht]
    \centering
    \ificml
    \includegraphics[width=\linewidth]{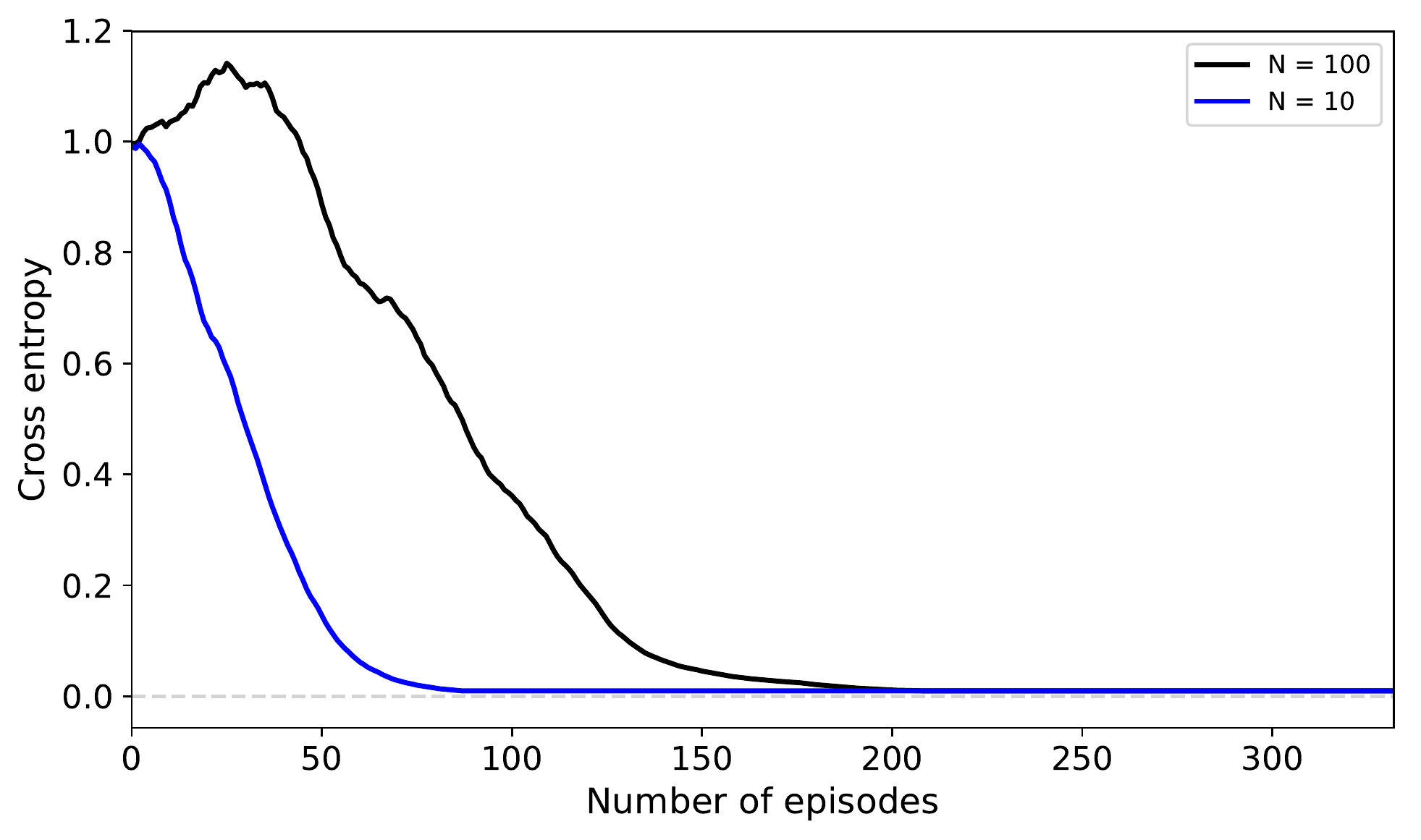}
    \else
    \includegraphics[width=0.6\linewidth]{MLP-discrete-categorical.pdf}
    \fi
    \caption{Cross entropy between the ground-truth SCM structure and the learned SCM structure. Each MLP is trained to predict the conditional distribution associated with each discrete variable
    with $N$ categories ($10$ and $100$ in this case), given its parents. Between 50 and 100 meta-examples are sufficient to recover the causal structure.}
    \label{fig:mlp_scm}
\end{figure}

In Appendix~\ref{sec:continuous}, we consider the case of continuous
scalar multimodal variables. The ground truth joint distribution is obtained by making the effect $B$ a non-linear function $f(A)$ of cause $A$,
where $f$ is a randomly generated spline. Figure~\ref{fig:encoder_scm}
shows an example of a resulting joint distribution. We model the
conditionals with mixture density networks~\citep{bishop1994mixture}
and the marginals by a Gaussian mixture. We obtain results that are similar
to the discrete case, with the correct causal interpretation
being recovered quickly, as illustrated in Figure~\ref{fig:encoder_gamma_evo}.

We also show in Appendix~\ref{sec:linear-gaussian} results on models with two continuous random variables, each being distributed as a multivariate Gaussian, with $N = 10$ dimensions. Similar to the experiment with discrete random variables, the same argument about parameter counting mentioned in Section~\ref{sec:parameter-counting} holds here. Again, we obtain results consistent with the previous examples, where the structural parameter $\gamma$ converges to $1$, effectively recovering the correct causal model $A \rightarrow B$.

\section{Representation Learning}

So far, we have assumed that the system has unrestricted access to the true underlying causal variables, $A$ and $B$. However in many realistic scenarios for learning agents, the observations available to the learner might not be instances of the true causal variables but sensory-level data instead, like pixels and sounds. If this is the case, our working assumption -- that the correct causal graph will be sparsely connected, made of independent components, and affected sparsely by distributional shifts -- can not be expected to hold true in general in the space of observed variables. To tackle this, we propose to follow the deep learning objective of disentangling the underlying causal variables~\citep{Bengio-Courville-Vincent-TPAMI2013}, and learn a representation in which these properties hold. In the simplest form of this setting, the learner must map its raw observations to a hidden representation space $H$ via an encoder $\mathcal E$. The encoder is trained such that the hidden space $H$ helps to optimize the meta-transfer objective described above, i.e., we consider the encoder, along with $\gamma$, as part of the set of structural or meta-parameters to be optimized with respect to the meta-transfer objective.

To study this simplified setting, we consider that our raw observations $(X, Y)$ originate from the true causal variables $(A, B)$ via the action of a ground truth decoder $\mathcal D$ (or generator network) that the learner is not aware of but is implicitly trying to invert, as illustrated in Figure~\ref{fig:encoder_arch}. The variables $A$, $B$, $X$ and $Y$ are assumed to be scalars, and we first consider $\mathcal D$ be a rotation matrix such that:
\begin{equation}
\begin{bmatrix}
X \\ Y
\end{bmatrix} = R(\theta_{\mathcal D}) 
\begin{bmatrix}
A \\ B
\end{bmatrix}
\end{equation}
The encoder is set to another rotation matrix, one that maps the observations $X, Y$ to the hidden representation $U, V$ as follows: 
\begin{equation}
\begin{bmatrix}
U \\ V
\end{bmatrix} = R(\theta_{\mathcal E}) 
\begin{bmatrix}
X \\ Y
\end{bmatrix}
\end{equation}
 The causal modules are now to be trained on the variables $U$ and $V$ in the same way as detailed in Section~\ref{sec:two-observed-variables}, as if they were observed directly. Indeed, if the encoder is valid one would obtain either $(U, V) = (A, B)$ or $(U, V) = (B, A)$ up to a negative sign, but we say in that case and without loss of generality that $(U, V)$ recovered $(A, B)$, corresponding to the solution $\theta_{\mathcal{E}} = -\theta_{\mathcal{D}}$. In this case, the model $U \to V$ is causal and should therefore have an advantage over the anticausal model $V \to U$, as far as adaptation speed on the transfer distribution is concerned. However, if the encoder is not valid, one would obtain superpositions of the form:
\begin{align}
U &= \cos(\theta) A - \sin(\theta) B\\
V &= \sin(\theta) A + \cos(\theta) B
\end{align}
where $\theta = \theta_{\mathcal{E}} + \theta_{\mathcal{D}}$. In the extremum where $\theta = \frac{\pi}{4}$, it is clear that the model $U \to V$ will not have an advantage over the model $V \to U$ in terms of regret on the transfer distribution. However, the question we are interested in is whether it is possible to learn the encoder $\theta_{\mathcal{E}}$. 
We verify this experimentally using Algorithm~\ref{alg:main}, but where the meta-parameters are now both $\gamma$ (choosing between cause and effect which is which) and the parameters
of the encoder (here the angle of a rotation matrix).
The details of that experiment are provided in Appendix~\ref{sec:encoder-experiment}, which illustrates -- see Figure~\ref{fig:encoder_evo} -- how the proposed objective can disentangle (here in a very simple setting) the ground truth variables (up to permutation).

\begin{figure}[ht]
    \centering
    \ificml
    \includegraphics[width=\linewidth]{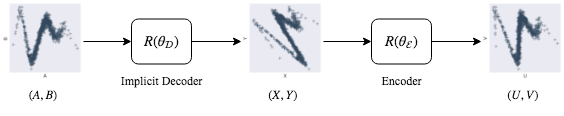}
    \vspace*{-5mm}
    \else
    \includegraphics[width=0.5\linewidth]{enc_diag_new.png}
    \fi
    \caption{The complete computational graph. The variables $(A, B)$ are assumed to originate from the true underlying causal distribution, but the observations available to the learner are $(X, Y)$ samples, which are obtained from $(A, B)$ via the action of an implicit (a priori unknown) decoder $R(\theta_{\mathcal{D}})$. The encoder $R({\theta_{\mathcal{E}}})$ must be learned to undo the action of the (unknown) decoder and thereby recover the true causal variables.}
    \vspace*{-3mm}
    \label{fig:encoder_arch}
\end{figure}

\begin{figure}[ht]
    \centering
    \ificml
    \includegraphics[width=\linewidth]{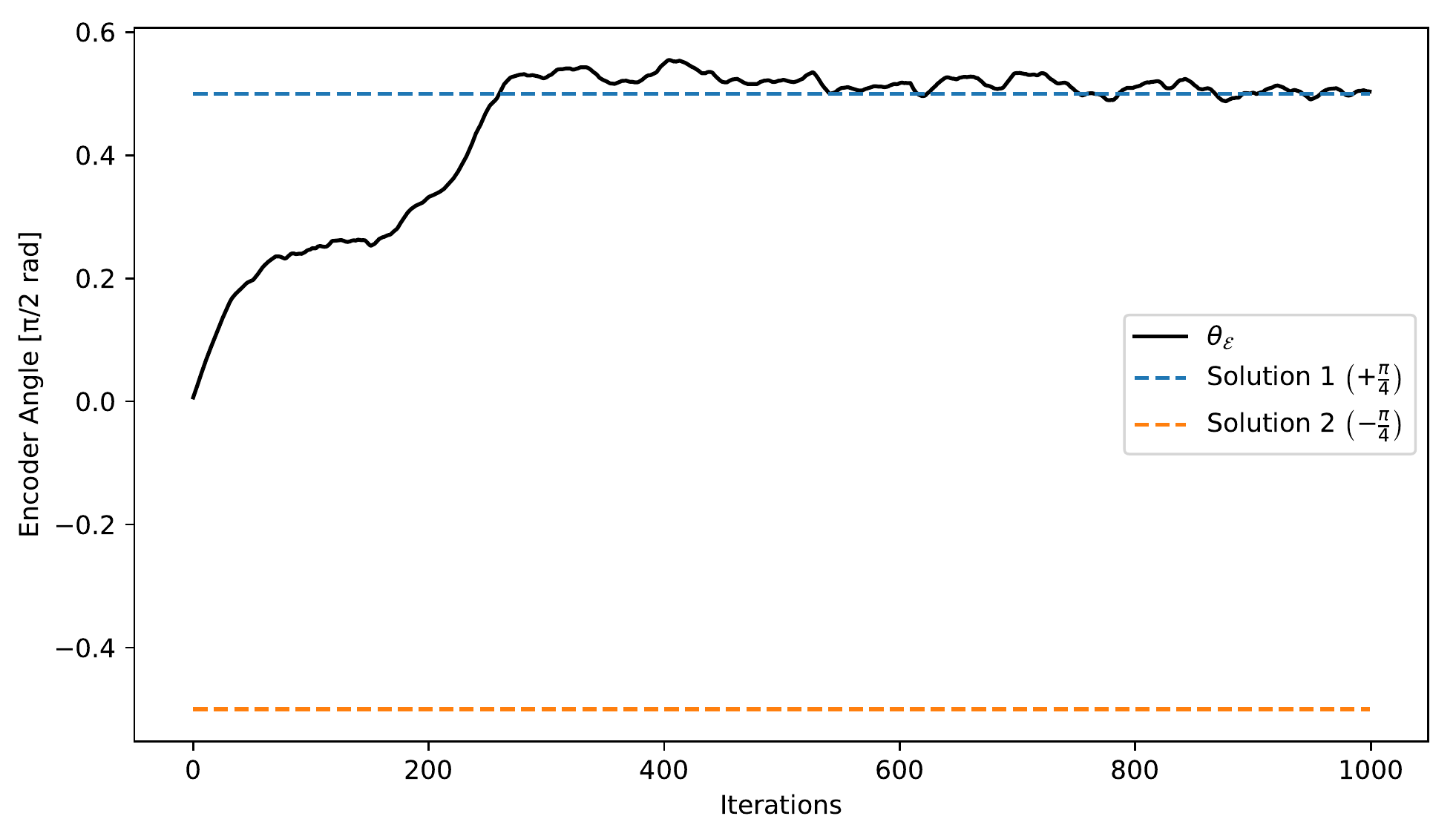}
    \vspace*{-5mm}
    \else
    \includegraphics[width=0.6\linewidth]{theta-evo.pdf}
    \fi
    \caption{Evolution of the encoder parameter $\theta_{\mathcal{E}}$ as training progresses. We set the parameter $\theta_{\mathcal{D}}$ of the implicit decoder to $-\frac{\pi}{4}$: this corresponds to two valid solutions for the encoder, namely $+\frac{\pi}{4}$ and $-\frac{\pi}{4}$. For the former, we obtain $\theta = 0$ corresponding to the correct causal graph $U \to V$; and for the latter, $\theta = \frac{\pi}{2}$, corresponding to the correct causal graph $V \to U$. The encoder parameter $\theta_{\mathcal{E}}$ is trained jointly with the structural parameter, and we find that it converges to one of the two valid solutions. Further details and the corresponding evolution of the structural parameter can be found in Appendix~\ref{sec:encoder-experiment}. 
    }
    \vspace*{-3mm}
    \label{fig:encoder_evo}
\end{figure}

\section{Related Work}


Although this paper focuses on the causal graph, the proposed objective is motivated by the more general question of discovering the underlying causal variables (and their dependencies) that explain the environment of a learner and make it possible for that learner to plan appropriately. The discovery of underlying explanatory variables has come under different names, in particular the notion of disentangling underlying variables~\citep{Bengio-Courville-Vincent-TPAMI2013}. As stated already by \citet{Bengio-Courville-Vincent-TPAMI2013} and clearly demonstrated by ~\citet{Locatello2018ChallengingCA},
assumptions, priors or biases are necessary to identify the underlying explanatory variables. The latter paper~\citep{Locatello2018ChallengingCA} also reviews and evaluates recent work on disentangling, and discusses different metrics that have been proposed. An extreme view of disentangling is that
the explanatory variables should be marginally independent, and many deep generative models~\citep{Goodfellow-et-al-2016-Book} and independent component analysis models~\citep{Hyvarinen-2001,Hyvarinen2018NonlinearIU} are built on this assumption. However, the kinds of high-level variables that we manipulate with natural language are not marginally independent: they are related to each other through statements that are usually expressed in sentences (e.g., a classical symbolic AI fact or rule), involving only a few concepts at a time. This kind of assumption has been proposed to help discover such linguistically relevant high-level representations from raw observations, such as the 
consciousness prior~\citep{Bengio-consciousness-arxiv2017}, with the idea that humans focus at any particular
time on just a few concepts that are present to our consciousness. The work presented here could provide an interesting meta-learning objective to help learn such encoders as well as figure out how the resulting variables are related to each other. In that case, one should distinguish two important assumptions: the first one is that the causal graph is sparse (has few edges, as in the consciousness prior~\citep{Bengio-consciousness-arxiv2017} and in some methods to learn Bayes net structure, e.g.~\citep{Schmidt2007LearningGM}); and the second one is that it changes sparsely due to interventions (which is the focus of this work). 
Approaches for Bayesian network structure learning based on discrete search over model structures and simulated annealing are reviewed in \citet{heckerman1995learning}. There, it has been common to use Minimum Description Length (MDL) principles 
to score and search over models \citet{lam1993using,friedman1998learning}, or the Bayesian Information Criterion (BIC) to search for models with high relative posterior probability \cite{heckerman1995learning}. Prior work such as \citet{heckerman1995learning} has also relied upon purely observational data, without the possibility of interventions and therefore focused on learning likelihood or hypothesis equivalence classes for network structures. 
Since then, numerous methods have also been devised to infer the causal direction from purely observational data~\citep{peters2017elements}, based on specific, generally parametric assumptions, on the underlying causal graph.
Pearl's seminal work on do-calculus \cite{pearl1995causal,pearl2009causality,bareinboim2016causal} lays a foundation for expressing the impact of interventions on probabilistic graphical models -- we use it in our work.
In contrast, here we are proposing a meta-learning objective function for learning causal structure, not requiring any specific constraints on causal graph structure, only on the sparsity of the changes in distribution in the correct causal graph parametrization.

Our work is also related to other recent advances in causation, domain adaptation and transfer learning.  
\citet{magliacane2018domain} have sought to identify a subset of features that lead to the best predictions for a variable of interest in a source domain such that the conditional distribution of the variable of interest given these features is the same in the target domain. \citet{johansson2016learning} examine counterfactual inference and formulate it as a domain adaptation problem. 
\citet{shalit2017estimating} propose a technique called counterfactual regression for estimating individual treatment effects from observational data. 
\citet{rojas2018invariant} propose a method to find an  optimal subset that makes the target
independent from the selection variables. To do so, they make the assumption that if the
conditional distribution of the target given some subset is invariant across different
source domains, then this conditional distribution must also be the same in the target domain. 
\citet{parascandolo2017learning} propose an algorithm to recover a set of independent causal  mechanisms
by establishing competition between mechanisms, hence driving specialization. \citet{alet2018modular} proposed a meta learing algorithm to recover a set of specialized modules, but did not establish any connections to causal mechanisms. More recently, \citet{dasgupta2019causal} adopted a meta-learning approach to draw causal inferences from purely observational data. 





\section{Conclusion and Future Work}

We have established in very simple bivariate settings that the rate at which a learner adapts to sparse changes in the distribution of observed data can be exploited to select or optimize causal structure and disentangle the causal variables. This relies on the assumption that with the correct causal structure, those distributional changes are localized and sparse. We have demonstrated these ideas through theoretical results as well as experimental validation. See
\ificml
\url{https://github.com/ec6dde01667145e58de60f864e05a4/CausalOptimizationAnon}
\else
\url{https://github.com/authors-1901-10912/A-Meta-Transfer-Objective-For-Learning-To-Disentangle-Causal-Mechanisms}
\fi
for 
\ificml
anonymized
\fi
source code of the experiments.

This work is only a first step in the direction of optimizing causal structure based on the speed of adaptation to modified distributions. On the experimental side, many settings other than those studied here should be considered, with different kinds of parametrizations, richer and larger causal graphs, different kinds of optimization procedures, etc. Also, much more needs to be done in exploring how the proposed ideas can be used to learn good representations in which the causal variables are disentangled, since we have only experimented at this point with the simplest possible encoder with a single degree of freedom. Scaling up these ideas would permit their application towards improving the way in which learning agents deal with non-stationarities, and thus improving sample complexity
and robustness of learning agents.

\ificml
\else
\section*{Acknowledgements}
We would like to acknowledge support for this project from NSERC, CIFAR and Canada Research Chairs, as well as the feedback from Rémi Le Priol, Isabelle Lacroix, Alexandre Piché, and Akram Erraqabi. AG would like to thank Sergey Levine, Chelsea Finn, Michael Chang, Abhishek Gupta for useful discussions.
\fi


\vskip 0.2in
\bibliography{causality}
\bibliographystyle{icml2019}

\clearpage
\newpage

\appendix
\counterwithin{figure}{section} 
\counterwithin{table}{section}

\ificml
\begin{strip}
\begin{center}
    SUPPLEMENTARY MATERIAL
    \vspace*{1cm}
\end{center}
\end{strip}
\fi

\section{Results on Non-Identifiability of Causal Structure}
\label{sec:non-identifiability}
We show here that the maximum likelihood estimation of both models specified in Equation \eqref{eq:bivariate-models} yields the same estimated distribution over $A$ and $B$, i.e., the joint likelihood on the training distribution is not sufficient to distinguish the $A \rightarrow B$ and $B\rightarrow A$ causal models, in the non-parametric case (no assumption at all on the family of distributions). Let
\begin{align*}
    \theta_{i} &= P_{A \rightarrow B}(A = i) & \theta_{j|i} &= P_{A \rightarrow B}(B=j\mid A=i)\\
    \eta_{j} &=P_{B \rightarrow A}(B=j) & \eta_{i|j}&=P_{B \rightarrow A}(A=i\mid B=j).
\end{align*}
We now state the maximum likelihood estimators for each models:
\begin{align}
  \hat{\theta}_i &= n_i / n &\hat{\theta}_{j|i} &= n_{ij} / n_i \nonumber\\
  \hat{\eta}_{j} &= n_j / n &\hat{\eta}_{i|j} &= n_{ij} / n_j
\end{align}
where $n$ is the total number of observations, $n_i$ the number of times we observed $A=i$, $n_j$ the number of times we observed $B=j$ and $n_{ij}$ the number of times we observed $A=i$ and $B=j$ jointly. We can now compute the likelihood for each model:
\begin{align}
  \hat{P}_{A \rightarrow B}(A, B) &= \hat{\theta}_i \hat{\theta}_{j|i} = n_{ij}/n \nonumber\\
  \hat{P}_{B \rightarrow A}(A, B) &= \hat{\eta}_j \hat{\eta}_{i|j} = n_{ij}/n
\end{align}
which is what we intended to show. To illustrate this result, we also experiment with learning the modules for both models $A \rightarrow B$ and $B \rightarrow A$ with SGD. In Figure~\ref{fig:generalization-same-distribution}, we show the difference in log-likelihoods between these two models, evaluated on training and test data sampled from the same distribution, during training. We can see that while the model $A \rightarrow B$ fits the data faster than the other model (corresponding to a positive difference in Figure~\ref{fig:generalization-same-distribution}), both models achieve the same log-likelihoods on both models at convergence. This shows that the two models are indistinguishable based on data sampled from the same distribution, even on test data.

\begin{figure*}
    \centering
    \includegraphics[width=0.99\linewidth]{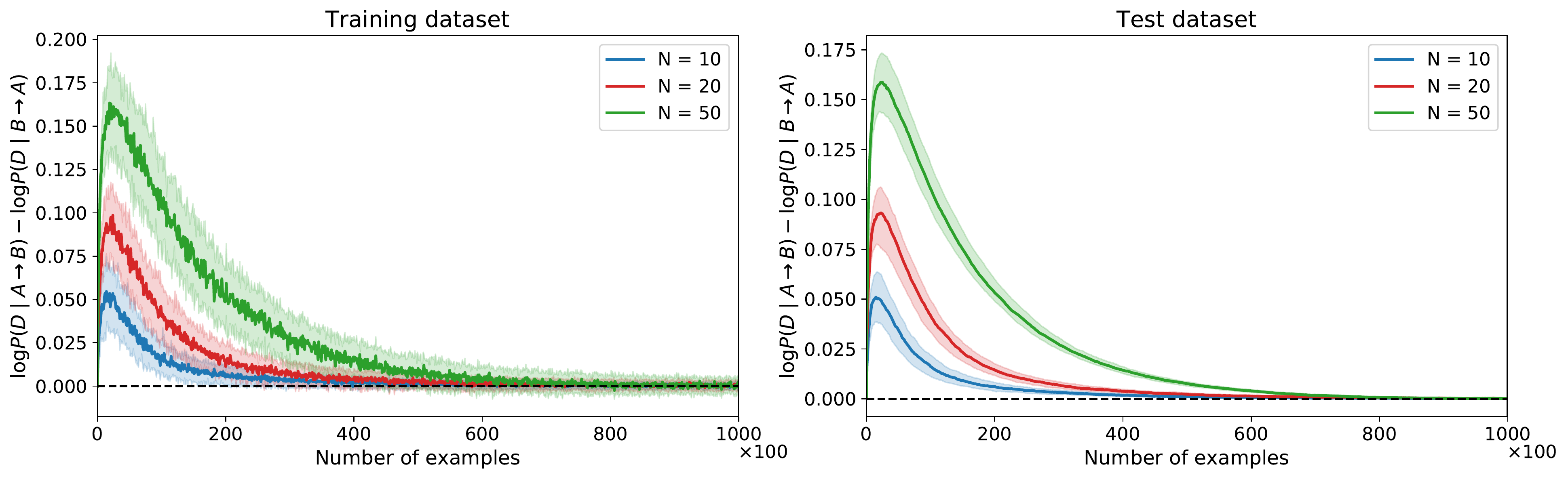}
    \caption{Difference in log-likelihoods between the two models $A \rightarrow B$ and $B \rightarrow A$ on training and test data from the same distribution on discrete data, for different values of $N$, the number of discrete values per variable. Once fully trained, both models become indistinguishable from their log-likelihoods only, even on test data. The solid curves represent the median values over 100 different runs, and the shaded areas their 25-75 quantiles.}
    \label{fig:generalization-same-distribution}
\end{figure*}

\ifplanB
\else
\section{Adaptation to the transfer distribution}
\label{app:adaptation-transfer}
The adaptation with only a few gradient steps on data coming from a different, but related, transfer distribution is critical in getting a signal that can be leveraged by our meta-learning algorithm. To show the effect of this adaptation, and motivate our use of only a small amount of data from the transfer distribution, we experimented with a model on discrete random variables taking $N = 10$ possible values.

In this experiment, we fixed the underlying causal model to be $A \rightarrow B$, and trained the modules for each marginal and conditional distributions with maximum likelihood on a large amount of data from some training distribution, as explained in Appendix~\ref{sec:non-identifiability}. See also Appendix~\ref{sec:bivariate-experiment-discrete} and Table~\ref{tab:categorical-modules} for details on the definitions of these modules.

We then adapt all the modules on data coming from a transfer distribution, corresponding on an intervention on the random variable $A$ (ie. the marginal $P(A)$ of the ground truth model is modified, while leaving $P(B\mid A)$ fixed). We used RMSprop for the adaptation, with the same learning rate. The experiment was repeated over 100 different training distributions, and over 100 transfer distributions for each training distributions, leading to 10\,000 experiments overall. The procedure to acquire different training/transfer distributions is details in Appendix~\ref{sec:bivariate-experiment-discrete}.

In Figure \ref{fig:adaptation-curve}, we report the log-likelihoods of both models, evaluated on a large test set of 10\,000 from the transfer distribution. We can see that as the number of examples from the transfer distribution (equal to the number of adaptation steps) increases, the two models eventually reach the same log-likelihood, reflecting our observations from Appendix~\ref{sec:non-identifiability}. However the causal model $A \rightarrow B$ adapts faster than the other model $B \rightarrow A$, with the most informative part of the trajectory (where the difference is the largest) is within the first 10 to 20 examples.
\fi

\section{Proof of the Zero-Gradient Proposition}
\label{sec:proof-zero-gradient}
\setcounter{prop}{\getrefnumber{prop:zero-gradient}}
\addtocounter{prop}{-1}

Let us restate more formally and prove Proposition~\ref{prop:zero-gradient}. 
\ifplanB
\else
We use the same notation as in Section~\ref{sec:many-factors}.
\fi
\begin{prop}
Consider conditional probability modules $P_{\theta_i}(V_i | {\rm pa}(i,V,B_i))$ where $B_{ij}=1$ indicates that $V_j$ is among the parents ${\rm pa}(i,V,B_i)$ of $V_i$ in a directed acyclic causal graph.
Consider ground truth training distribution $P_1$ and transfer distribution $P_2$ over these variables, and ground truth causal structure $B$. The joint log-likelihood $\mathcal{L}(V)$ for a sample $V$ with respect to the module parameters $\theta$ decomposed into module parameters $\theta_i$ is \mbox{$\mathcal{L}(V)=\sum_i \log P_{\theta_i}(V_i | {\rm pa}(i,V,B_i))$.} 
If (a) a model has the correct causal structure $B$, and (b) it been trained perfectly on $P_1$, leading to estimated parameters $\theta$, and (c) the ground truth $P_1$ and $P_2$ only differ from each other only for some $P(V_i | {\rm pa}(i,V,B_i))$ for $i \in C$, then $E_{V\sim P_2}[\frac{\partial \mathcal{L}(V)}{\partial \theta_i}]=0$ for $i \notin C$.
\end{prop}
\begin{proof}
Let $V_{-i}$ be the subset of $V$ excluding $V_i$. We can simplify the expected gradient as follows.
\begin{align}
  E_{V\sim P_2}\bigg[&\frac{\partial \mathcal{L}(V)}{\partial \theta_i}\bigg] = \nonumber \\
      &\sum_V P_2(V) \sum_k \frac{\partial}{\partial \theta_i} \log P_{\theta_k}(V_k | {\rm pa}(k,V,B_k)) \nonumber \\
      &= \1_{i \in C} \sum_{V_{-i}} P_2(V_{-i}) \sum_{V_i} P_2(V_i|{\rm pa}(i,V,B_i)) \nonumber \\ &\;\;\frac{\partial}{\partial \theta_i} \log P_{\theta_i}(V_i | {\rm pa}(i,V,B_i))\ + \nonumber \\
      & \1_{i \notin C} \sum_{V_{-i}} P_2(V_{-i}) \sum_{V_i} P_1 (V_i|{\rm pa}(i,V,B_i)) \nonumber \\
      & \frac{\partial}{\partial \theta_i} \log P_{\theta_i}(V_i | {\rm pa}(i,V,B_i)) \nonumber \\
\end{align}
where the second equality is obtained because $\theta_i$ does not influence module $k\neq i$, and $P_2$ is the same $P_1$ for conditionals with $i \notin C$ (assumption (c)). Now for the special case of $i \notin C$, we obtain
\begin{align}
  E_{V\sim P_2}\left[\frac{\partial \mathcal{L}(V)}{\partial \theta_i}\right] &= \sum_{V_{-i}} P_2(V_{-i}) \sum_{V_i} P_1 (V_i|{\rm pa}(i,V,B_i))\nonumber\\ & \qquad\frac{\partial}{\partial \theta_i} \log P_{\theta_i}(V_i | {\rm pa}(i,V,B_i)) \nonumber \\
  &= \sum_{V_{-i}} P_2(V_{-i}) \sum_{V_i} P_{\theta_i} (V_i|{\rm pa}(i,V,B_i)) \nonumber \\ &\qquad\frac{\partial}{\partial \theta_i} \log P_{\theta_i}(V_i | {\rm pa}(i,V,B_i)) \nonumber \\
  &= 0
\end{align}
where the second equality arises from assumption (b), and the last line from zeroing the inner sum via the general identity 
$$\sum_v p_\theta(v) \frac{\partial}{\partial \theta} \log p_\theta(v)=\frac{\partial}{\partial \theta} \sum_v p_\theta(v)=\frac{\partial 1}{\partial \theta}=0.$$
\end{proof}

\section{Pseudo-Code}
\label{sec:pseudo-code}

\SetKwRepeat{Repeat}{Repeat}{}
\DontPrintSemicolon
\begin{algorithm}[H]
\caption{Meta-Transfer Learning of Causal Structure}
\label{alg:main}
\SetAlgoVlined 
    Draw initial meta-parameters of learner\\
    Draw a training set from training distr.\\
    Set causal structure to include all edges \\
    Initialize learner parameters for this model\\
    Pre-train the learner's parameters on the training set\\
    \Repeat( $J$ times){}{
      Draw a transfer distr.\\
      Draw causal structure(s) according to meta-parameters\\
      \Repeat( $T$ times){}{
          Sample minibatch from transfer distribution\\
          Accumulate online log-likelihood of minibatch\\
          Update the model parameters accordingly}
        Compute the meta-parameters gradient estimator\\
        Update the meta-parameters by SGD\\
        Optionally reset parameters to pre-training value}
\end{algorithm}

\section{Proof of the Structural Parameter Gradient Proposition}
\label{sec:proof-posterior}

\setcounter{prop}{\getrefnumber{prop:posterior}}
\addtocounter{prop}{-1}

Let us restate more formally and prove Proposition~\ref{prop:posterior}. 

\begin{prop}
The gradient of the negative log-likelihood regret of the transfer data 
$$\regret = -\log \left[ \sigm(\gamma) \mathcal{L}_{A \rightarrow B} + (1-\sigm(\gamma)) \mathcal{L}_{B \rightarrow A} \right]$$
with respect to the structural parameter $\gamma$ (where $\sigma(\gamma)=P(A\rightarrow B)$) is given by
\begin{equation}
    \frac{\partial \regret}{\partial \gamma} = \sigma(\gamma) - P(A \rightarrow B \mid D_{2}),
\end{equation}
where $D_{2}$ is the transfer data, and $P(A \rightarrow B\mid D_{2})$ is the posterior probability of the hypothesis $A \rightarrow B$ (when the alternative is $B \rightarrow A$), defined by applying Bayes rule to $P(D_2 \mid A \rightarrow B)=\prod_{t=1}^T P(a_t, b_t | A \rightarrow B, \theta_{t}) = \mathcal{L}_{A \rightarrow B}$. Furthermore, this can be equivalently written as
\begin{equation}
    \label{eq:proof-gradient-score}
    \frac{\partial \regret}{\partial \gamma} = \sigma(\gamma) - \sigma(\gamma + \Delta),
\end{equation}
where $\Delta = \log \mathcal{L}_{A \rightarrow B} - \log \mathcal{L}_{B \rightarrow A}$ is the difference between the log-likelihoods of the two hypotheses on transfer data $D_{2}$.
\end{prop}
\begin{proof}
First note that, using Bayes rule,
\begin{align}
  &P(A \rightarrow B\mid D_2) =\nonumber\\
  &\frac{P(D_2 \mid A \rightarrow B) P(A \rightarrow B)}{P(D_2 \mid A \rightarrow B) P(A \rightarrow B) + P(D_2 \mid B \rightarrow A) P(B \rightarrow A)} \nonumber \\
  &= \frac{\mathcal{L}_{A\rightarrow B} \sigma(\gamma)}{\mathcal{L}_{A \rightarrow B} \sigma(\gamma) + \mathcal{L}_{B \rightarrow A} (1-\sigma(\gamma))} &\nonumber \\
  &= \frac{\sigma(\gamma) \mathcal{L}_{A \rightarrow B}}{M},
  \label{eq:proof-gradient-posterior}
\end{align}
where $M = \sigma(\gamma)\mathcal{L}_{A \rightarrow B} + (1 - \sigma(\gamma))\mathcal{L}_{B \rightarrow A}$ is the online likelihood of the transfer data under the mixture, so that the regret is $\regret = -\log M$. For the second line above, note that
\begin{align}
    P(D_2 | A\rightarrow B) &= \prod_{t=1}^T P(a_t, b_t | A \rightarrow B, \{(a_s,b_s)\}_{s=1}^{t-1}) \nonumber \\
    & = \prod_{t=1}^T P(a_t, b_t | A \rightarrow B, \theta_{t}) = \mathcal{L}_{A \rightarrow B}
\end{align}
where $\theta_t$ encapsulates the information about $\{(a_s,b_s)\}_{s=1}^{t-1})$ (through some adaptation procedure). 
Since we only consider the two hypotheses $A \rightarrow B$ and $B \rightarrow A$, we also have $P(B \rightarrow A\mid D_2)=\frac{(1-\sigma(\gamma))\mathcal{L}_{B\rightarrow A}}{M}=1-P(A\rightarrow B\mid D_2)$. Then
\begin{align}
  \frac{\partial \regret}{\partial \gamma} &= -\frac{\sigma(\gamma)(1-\sigma(\gamma)) \mathcal{L}_{A \rightarrow B} - \sigma(\gamma)(1-\sigma(\gamma)) \mathcal{L}_{B \rightarrow A}}{M} \nonumber \\
  &= \sigma(\gamma) P(B \rightarrow A\mid D_2) \nonumber\\
  &\qquad - (1-\sigma(\gamma)) P(A \rightarrow B\mid D_2) \nonumber\\
  &= \sigma(\gamma) + \sigma(\gamma) P(A \rightarrow B \mid D_2) \nonumber \\
  &\qquad - P(A \rightarrow B\mid D_2) - \sigma(\gamma) P(A \rightarrow B\mid D_2) \nonumber\\
  &= \sigma(\gamma) - P(A \rightarrow B\mid D_2)
\end{align}
which concludes the first part of the proof. Moreover, in order to prove the equivalent formulation in Equation \eqref{eq:proof-gradient-score}, it is sufficient to prove that $P(A\rightarrow B\mid D_{2}) = \sigma(\gamma + \Delta)$. Using the logit function $\sigma^{-1}(z) = \log \frac{z}{1 - z}$, and the expression in Equation~\eqref{eq:proof-gradient-posterior}, we have
\begin{align}
  \sigma^{-1}(P(A\rightarrow B\mid D_{2})) &= \log \frac{\sigma(\gamma)\mathcal{L}_{A \rightarrow B}}{M - \sigma(\gamma)\mathcal{L}_{A \rightarrow B}} \nonumber \\
  &= \log \frac{\sigma(\gamma)\mathcal{L}_{A \rightarrow B}}{(1 - \sigma(\gamma))\mathcal{L}_{B \rightarrow A}}\nonumber\\
  &=  \underbrace{\log \frac{\sigma(\gamma)}{1 - \sigma(\gamma)}}_{=\,\gamma}\ + \nonumber\\
  &\qquad\underbrace{\log \mathcal{L}_{A \rightarrow B} - \log \mathcal{L}_{B \rightarrow A}}_{=\,\Delta}\nonumber \\
  &= \gamma + \Delta
\end{align}
\end{proof}

\section{Proof of the Proposition on the Convergence Point of Gradient Descent on the Structural Parameter}
\label{sec:proof-convergence}
\setcounter{prop}{\getrefnumber{prop:convergence}}
\addtocounter{prop}{-1}
We use the same notation as in the above proof and statement.
\begin{prop}
Stochastic gradient descent (with appropriately decreasing learning rate) on $E_{D_2}[\regret]$, with \mbox{$\regret=-\log \left[ \sigm(\gamma) \mathcal{L}_{A \rightarrow B} + (1-\sigm(\gamma)) \mathcal{L}_{B \rightarrow A} \right]$} and with steps following $\frac{\partial \regret}{\partial \gamma}$ converges towards $\sigm(\gamma)=1$ if $E_{D_2}[\log \mathcal{L}_{A\rightarrow B}]>E_{D_2}[\log \mathcal{L}_{B \rightarrow A}]$, or $\sigma(\gamma)=0$ otherwise.
\end{prop}
\begin{proof}
We are going to consider the fixed point of gradient descent when the gradient is zero, since we already know that SGD converges with an appropriately decreasing learning rate. Let us introduce some notation to simplify the algebra: $p=\sigm(\gamma)$, $M=p \mathcal{L}_{A \rightarrow B} + (1-p) \mathcal{L}_{B \rightarrow A}$, so $\regret=\log M$, and define $P_1=\frac{p \mathcal{L}_{A \rightarrow B}}{M}=P(A\rightarrow B\mid D_2)$, and $P_2=\frac{(1-p) \mathcal{L}_{B \rightarrow A}}{M}=1-P_1$. Framing the stationary point in terms of $p$ rather than $\gamma$ gives us the inequality constraints $-p\leq 0$ and $p-1\leq 0$ and no equality constraint. Applying the KKT conditions with constraint functions $-p$ and $p-1$ gives us
\begin{eqnarray}
  E_{D_2}\left[\frac{\partial \regret}{\partial p}\right] &=& - \mu_1 + \mu_2 \nonumber \\
  \mu_i &\geq& 0 \nonumber \\
  \mu_1 p &=& 0 \nonumber \\
  \mu_2 (p-1) &=& 0
\end{eqnarray}
We already see from the last two equations that if $p \in (0,1)$ (i.e. excluding 0 and 1), we must have $\mu_1=\mu_2=0$, i.e., $E[\frac{\partial \regret}{\partial p}]=0$ (with drop the $D_2$ subscript on $E$ when it is clear from context). Let us study that case first and show that it leads to an inconsistent set of equations (thus forcing the solution to be either $p=0$ or $p=1$).
Let us rewrite the gradient to highlight $p$ in it:
\begin{align}
  \frac{\partial \regret}{\partial p}=&P(A\rightarrow B\mid D_2) - p \nonumber\\
  =& \frac{p \mathcal{L}_{A \rightarrow B}}{p \mathcal{L}_{A \rightarrow B} + (1-p) \mathcal{L}_{B \rightarrow A}} - p \nonumber \\
  =& \frac{p \mathcal{L}_{A \rightarrow B} - p(p \mathcal{L}_{A \rightarrow B} + (1-p) \mathcal{L}_{B \rightarrow A})}{M}\nonumber\\
  =& \frac{p(1-p)\mathcal{L}_{A \rightarrow B} - p(1-p) \mathcal{L}_{B \rightarrow A}}{M} \nonumber \\
  =& p(1-p)\frac{\mathcal{L}_{A \rightarrow B}-\mathcal{L}_{B \rightarrow A}}{M}
\end{align}
The KKT conditions with the above two inequality constraints for $0 \leq p \leq 1$ give
\begin{align}
  E\left[\frac{\partial \regret}{\partial p}\right] &= \mu_2 - \mu_1.
\end{align}
If we consider the solutions $p \in (0,1)$ (i.e., $\mu_1=\mu_2=0$) we now show that we get a contradiction. First note that
to satisfy the above equation with $\mu_1=\mu_2=0$ means that either $p=0$ or $p=1$
(which is inconsistent with the assumption that $p\in(0,1)$) or that $E\left[\frac{\mathcal{L}_{A \rightarrow B}-\mathcal{L}_{B \rightarrow A}}{M}\right]=0$.
Let us consider that equation, and since $p\neq 0$ and $p\neq 1$ we can either multiply
by $p$ or by $1-p$ on both sides. Assuming $p\neq 0$ and multiplying by $p$ gives
\begin{align}
  0 &= E\left[\frac{p(\mathcal{L}_{A \rightarrow B}-\mathcal{L}_{B \rightarrow A})}{M}\right] = E\left[ P_1 - \frac{p \mathcal{L}_{B \rightarrow A}}{M} \right] \nonumber \\
    &= E\left[ P_1 - \frac{\mathcal{L}_{B \rightarrow A} - \mathcal{L}_{B \rightarrow A} - p \mathcal{L}_{B \rightarrow A}}{M} \right] \nonumber \\
    &= E\left[ P_1 + \frac{\mathcal{L}_{B \rightarrow A}}{M} - P_2 \right] \nonumber \\
    &= E\left[ P_1 + \frac{\mathcal{L}_{B \rightarrow A}}{M} - (1 - P_1) \right] = E\left[ \frac{\mathcal{L}_{B \rightarrow A}}{M} - 1\right].
\end{align}
For this equation to be satisfied, we need $\mathcal{L}_{B \rightarrow A}=M$ all the time, since $\mathcal{L}_{B \rightarrow A}\leq M$ by construction. This would however correspond to $p=0$. Similarly, assuming $p\neq 1$ we can multiply the stationarity equation by $1-p$ and get
\begin{align}
  0 &= E\left[\frac{(1-p)(\mathcal{L}_{A \rightarrow B}-\mathcal{L}_{B \rightarrow A})}{M}\right] \nonumber\\
  &= E\left[\frac{(1-p)\mathcal{L}_{A \rightarrow B}}{M}-P_2\right] \nonumber\\
    &= E\left[\frac{\mathcal{L}_{A \rightarrow B}}{M}-P_1-(1-P_1)\right] = E\left[\frac{\mathcal{L}_{A \rightarrow B}}{M}-1\right]
\end{align}
Again, this can only be 0 if $\mathcal{L}_{A \rightarrow B}=M$ all the time, i.e., $p=1$. We conclude that the solutions $p\in(0,1)$ are not possible because they would lead to inconsistent conclusions, which leaves only $p=0$ or $p=1$. When $p=0$ we have $E[\regret]=E[\log \mathcal{L}_{A \rightarrow B}]$, and when $p=1$ we have $E[\regret]=E[\log \mathcal{L}_{B \rightarrow A}]$. Thus the minimum will be achieved at $p=1$ when $E_{D_2}[\log \mathcal{L}_{A \rightarrow B}]>E_{D_2}[\log \mathcal{L}_{B \rightarrow A}]$, or $p=0$ otherwise.
\end{proof}

\section{More Than Two Causal Hypotheses}
\label{sec:many-factors}


In this section, we consider one approach to 
generalize to more than two causal structures.
We consider $m$ variables, corresponding to 
$O(2^{m^2})$ possible causal graphs, since each variable $V_j$ could be (or not) a direct cause of any variable $V_i$, leading to $m^2$ binary decisions. Note that a causal graph can in principle have cycles (if time is not measured with sufficient precision), although having a directed acyclic graph allows a much simpler sampling procedure (ancestral sampling). In our experiments the ground truth graph will always be directed, to make sampling easier and faster, but the learning procedure will not directly assume that. Motivated by the mechanism independence assumption, we propose a heuristic to learn the causal graph in which we independently parametrize the binary probability $p_{ij}$ that $V_j$ is a parent (direct cause) of $V_i$. As was the case for Section~\ref{sec:two-observed-variables}, we parametrize this Binomial distribution via binary edges $B_{ij}$ that specify the graph structure:
\begin{align}
 B_{ij} &\sim {\rm Bernoulli}(p_{ij}), \nonumber \\
 P(B) &= \prod_{ij} P(B_{ij}).
\end{align}
where $ p_{ij} = \sigm(\gamma_{ij})$. Let us define the parents of $V_i$, given $B$, as the set of $V_j$'s such that $B_{ij}=1$:
\begin{equation}
 {\rm pa}(i,V,B_i) = \{ V_j \mid B_{ij}=1, \; j\neq i\}
\end{equation}
where $B_i$ is the bit vector with elements $B_{ij}$ (and $B_{ii}=0$ is ignored). Similarly, we could parametrize the causal graph with a structural causal model where some of the inputs (from variable $j$) of each function (for variable $i$) can be ignored with some probability $p_{ij}$:
\begin{eqnarray}
 V_i = f_i(\theta_i, B_i, V, N_i)
\end{eqnarray}
where $N_i$ is an independent noise source to generate $V_i$ and $f_i$ parametrizes the generator (as in a GAN), while not being allowed to use variable $V_j$ unless $B_{ij}=1$ (and of course not being allowed to use $V_i$). 
We can consider that $f_i$ is a kind of neural network
similar to the denoising auto-encoders or with
dropout on the input, where $B_i$ is a binary mask vector that prevents $f_i$ from using some of the $V_j$'s (for which $B_{ij}=0$). 

\ifoldloss
We propose to generalize the transfer log-likelihood to the more general case by exploiting this factorization heuristic, as follows, to obtain a computationally efficient objective:
\begin{eqnarray}
 \regret = -E_B \left[ \sum_{t=1}^T \sum_{i=1}^m \log P_{\theta_{i,t}}(V_i=v_{ti} | {\rm pa}(i,v_t,B_i)) \right]
\end{eqnarray}
where $\theta_{i,t}$ are the parameters of the $i$-th mechanism (the structural equation for variable $V_i$) at adaptation step $t$ of the transfer episode,  $v_t$ is the vector of observed values at transfer example $t$, $D_2=(v_1, v_2, \ldots v_T)$ is the observed transfer data for one episode, and learning will involve sampling many such episodes, each time with a different $D_2$ sampled from a different transfer distribution (derived as a small change from the training distribution).
\else
The conditional likelihood $P_{B_i}(V_i=v_{ti}\mid {\rm pa}(i,v_t,B_i))$ measures how well
the model that uses the incoming edges $B_i$ for node $i$ performs for example $v_t$.
We build a multiplicative (or exponentiated) form of regret by multiplying these
likelihoods as $\theta_t$ changes during an adaptation episode, for node $i$:
\begin{equation}
    \mathcal{L}_{B_i} = \prod_t P_{B_i}(V_i=v_{ti}\mid {\rm pa}(i,v_t,B_i)).
\end{equation}
The overall exponentiated regret for the given graph structure $B$ is
    $\mathcal{L}_B = \prod_i \mathcal{L}_{B_i}$.
Similarly to the bivariate case, we want to consider a mixture over all the possible graph structures,
but where each component must explain the whole adaptation sequence, thus we define as a loss
for the generalized multi-variable case
\begin{align}
\label{eq:regret}
    \regret &= - \log E_B[ \mathcal{L}_B ]
\end{align}
\fi
Note the expectation over the $2^{m^2}$ possible values of $B$, which is intractable. However, we can still get an efficient stochastic gradient estimator, which
can be computed separately for each node of the graph (with
samples arising only out of $B_i$, the incoming edges into $V_i$):
\ifoldloss
\begin{prop}
\label{prop:unbiased}
The expected value $\regret$ of the episode regret over $P(B)$ can be decoupled as follows into
independent components (which can be sampled separately and independently) as follows:
\begin{equation}
\label{eq:simplified-L}
    \regret = -\sum_k P(B_k)  \sum_t \regret_{kt}
\end{equation}
where $\regret_{kt}=\log P_{\theta_{k,t}}(V_k=v_{tk} \mid {\rm pa}(k,v_t,B_k))$.
Furthermore, an unbiased stochastic gradient estimator $g_{ij}$ of $\frac{\partial \regret}{\partial \gamma_{ij}}$ can be obtained as follows:
\begin{align}
\label{eq:sum-only-over-Bi}
 \frac{\partial \regret}{\partial \gamma_{ij}} &= E_{B_i}[g_{ij}] \\
 \label{eq:gij}
 g_{ij}&= (\sigm(\gamma_{ij})-B_{ij})\sum_{t=1}^T \log P_{\theta_{i,t}}(V_i=v_{ti} | {\rm pa}(i,v_t,B_i))
\end{align}
where the $g_{ij}$'s for different $j$'s are all obtained from a single draw of $B_i$, for a given draw of $D_2$.
\end{prop}
The proof of this proposition can be found in Appendix~\ref{sec:proof-unbiased}.
\else
\begin{prop}
\label{prop:biased}
The overall regret (Equation~\eqref{eq:regret}) rewrites
\begin{equation}
    \regret = - \sum_i \log \sum_{B_i} P(B_i) \mathcal{L}_{B_i}
\end{equation}
and if we are willing to consider multiple samples of $B$ in parallel, a biased but asymptotically unbiased (as the number $K$ of these samples $B^{(k)}$ increases to infinity) estimator of the gradient of the overall regret with respect to meta-parameters can be defined:
\begin{equation}
\label{eq:gij}
    g_{ij} = \frac{\sum_k (\sigma(\gamma_{ij})-B_{ij}^{(k)}) \mathcal{L}_{B_i}^{(k)}}{\sum_k \mathcal{L}_{B_i}^{(k)}}
\end{equation}
where the $^{(k)}$ index indicates the values obtained for the $k$-th draw of $B$.
\end{prop}
\fi

\begin{proof}
Recall that $\mathcal{L}_B = \prod_i \mathcal{L}_{B_i}$ so we can rewrite the regress loss as follows:
\begin{align}
    \regret &= - \log E_B[ \mathcal{L}_B ] \nonumber \\
            &= - \log \sum_B P(B) \mathcal{L}_B \nonumber \\
            &= - \log \sum_{B_1} \sum_{B_2} \ldots \sum_{B_M} \prod_i P(B_i) \mathcal{L}_{B_i} \nonumber \\
            &= - \log \prod_i \left( \sum_{B_i} P(B_i) \mathcal{L}_{B_i} \right) \nonumber \\
            &= - \sum_i \log \sum_{B_i} P(B_i) \mathcal{L}_{B_i}\end{align}
So the regret gradient on meta-parameters $\gamma_i$ of node $i$ is
\begin{align}
 \frac{\partial \regret}{\partial \gamma_i} &= - \frac{\sum_{B_i} P(B_i) \mathcal{L}_{B_i} \frac{\partial \log P(B_i)}{\partial \gamma_i}}{\sum_{B_i} P(B_i) \mathcal{L}_{B_i}}\nonumber \\
  &= - \frac{E_{B_i}[\mathcal{L}_{B_i} \frac{\partial \log P(B_i)}{\partial \gamma_i}]}{E_{B_i}[\mathcal{L}_{B_i}]}
\end{align}
Note that with the sigmoidal parametrization of $P(B_{ij})$, 
$$\log P(B_{ij})=B_{ij} \log \sigm(\gamma_{ij}) + (1-B_{ij}) \log (1- \sigm(\gamma_{ij}))$$ 
as in the cross-entropy loss. Its gradient can similarly be simplified to
\begin{align}
  \frac{\partial \log P(B_{ij})}{\partial \gamma_{ij}} &=
  \frac{B_{ij}}{\sigm(\gamma_{ij})} \sigm(\gamma_{ij})(1-\sigm(\gamma_{ij})) \nonumber \\
  &\;\;\;\;- \frac{(1-B_{ij})}{(1-\sigm(\gamma_{ij}))}\sigm(\gamma_{ij})(1-\sigm(\gamma_{ij}))) \nonumber \\
    &= B_{ij}-\sigm(\gamma_{ij})
\end{align} 
A biased but asymptotically unbiased estimator of $\frac{\partial \regret}{\partial \gamma_{ij}}$ is thus obtained by sampling $K$ graphs (over which the means below are run):
\begin{align}
    g_{ij} = \sum_k (\sigma(\gamma_{ij})-B_{ij}^{(k)}) \frac{\mathcal{L}_{B_i^{(k)}}}{\sum_{k'} \mathcal{L}_{B_i^{(k')}}}
\end{align}
where index $^{(k)}$ indicates the $k$-th draw of $B$, and we obtain a weighted sum of the individual binomial gradients weighted by the relative regret of each draw $B_i^{(k)}$ of $B_i$,
leading to Equation~\eqref{eq:gij}.
\end{proof}
This decomposition is good news because the loss is a sum of independent
terms, one per node $i$, depending only of $B_i$ and and similarly $g_{ij}$ only
depends on $B_i$ rather than the
full graph structure. We use the
estimator from Equation~\eqref{eq:gij} in the general pseudo-code for meta-transfer learning of causal structure displayed in Algorithm~\ref{alg:main}.
\fi

\section{Results on Learning which is Cause and which is Effect}
\label{sec:bivariate-experiment}
In order to assess the performance of our meta-learning algorithm, we applied it on generated data from three different domains: discrete random variables, multimodal continuous random variables and multivariate gaussian-distributed variables. In this section, we describe the setups for all three experiments, along with additional results to complement the results described in the main text. Note that in all these experiments, we fix the structure of the ground-truth to be $A \rightarrow B$, and only perform interventions on the cause $A$.

\subsection{Discrete variables and Two Causal Hypotheses}
\label{sec:bivariate-experiment-discrete}
We consider a bivariate model, where both random variables are sampled from a categorical distribution. The underlying ground-truth model can be described as
\begin{align}
    A &\sim \Categorical(\pi_{A})\nonumber\\
    B \mid A = a &\sim \Categorical(\pi_{B|a}),
\end{align}
with $\pi_{A}$ is a probability vector of size $N$, and $\pi_{B|a}$ is a probability vector of size $N$, which depends on the value of the variable $A$. In our experiment, each random variable can take one of $N = 10$ values. Since we are working with only two variables, the only two possible models are:
\begin{itemize}
    \item \emph{Model $A \rightarrow B$}: $P(A, B) = P(A)P(B\mid A)$
    \item \emph{Model $B \rightarrow A$}: $P(A, B) = P(B)P(A\mid B)$
\end{itemize}
We build 4 different modules, corresponding to the model of each possible marginal and conditional distribution. These modules' definition and their corresponding parameters are shown in Table~\ref{tab:categorical-modules}.

\begin{table*}[ht]
    \centering
    \begin{tabular}{ll|l|c|c}
        \hline
        & Distribution & Module & Parameters & Dimension\\
        \hline
        \multirow{2}{*}{\emph{Model $A \rightarrow B$}} & $P(A)$ & $\displaystyle P(x_{A} = i\,;\,\theta_{A}) = [\softmax(\theta_{A})]_{i}$ & $\theta_{A}$ & $N$\\
        &$P(B\mid A)$ & $\displaystyle P(x_{B} = j\mid x_{A} = i\,;\,\theta_{B|A}) = [\softmax(\theta_{B|A}(i))]_{j}$ & $\theta_{B|A}$ & $N^{2}$\\
        \hline
        \multirow{2}{*}{\emph{Model $B \rightarrow A$}} &$P(B)$ & $\displaystyle P(x_{B} = j\,;\,\theta_{B}) = [\softmax(\theta_B)]_{j}$ & $\theta_{B}$ & $N$\\
        &$P(A\mid B)$ & $\displaystyle P(x_{A} = i\mid x_{B} = j\,;\,\theta_{A|B}) = [\softmax(\theta_{A\mid B}(j))]_{i}$ & $\theta_{A|B}$ & $N^{2}$\\
        \hline
    \end{tabular}
    \caption{Description of the 2 models, with the parametrization of each module, for a bivariate model with discrete random variables. \emph{Model $A\rightarrow B$} and \emph{Model $B \rightarrow A$} both have the same number of parameters $N^{2} + N$.}
    \label{tab:categorical-modules}
\end{table*}

In order to get a set of initial parameters, we first train all 4 modules on a training distribution. This training distribution corresponds to a fixed choice of $\pi_{A}^{(1)}$ and $\pi_{B|a}$ (for all $N$ possible values of $a$). Note that the superscript in $\pi_{A}^{(1)}$ emphasizes the fact that this defines the distribution prior to intervention, with the mechanism $P(B\mid A)$ being unchanged by the intervention. These probability vectors are sampled randomly from a uniform Dirichlet distribution
\begin{align}
    \pi_{A}^{(1)} &\sim \Dirichlet(\mathbf{1}_{N})\nonumber\\
    \pi_{B|a} &\sim \Dirichlet(\mathbf{1}_{N})\qquad \forall a \in [1, N].
\end{align}

Given this initial training distribution, we can sample a large dataset of training examples $\{(a_{i}, b_{i})\}_{i=1}^{n}$ from the ground-truth model, using ancestral sampling.
\begin{align}
    a&\sim \Categorical(\pi_{A}^{(1)})\nonumber\\
    b&\sim \Categorical(\pi_{B|a}).
    \label{eq:app-ancestral-sampling}
\end{align}

Using this large dataset from the training distribution, we can train all 4 modules using gradient descent, or any other advanced first-order optimizer, like RMSprop. The parameters $\theta_{A}$, $\theta_{B|A}$, $\theta_{B}$ \& $\theta_{A|B}$ of the different modules found after this initial training will be used as the initial parameters for the adaptation on a new transfer distribution.

Similar to the way we defined the training distribution, we can define a transfer distribution as a soft intervention on the random variable $A$. In this experiment, this accounts for changing the distribution of $A$, that is with a new probability vector $\pi_{A}^{(2)}$, also sampled randomly from a uniform Dirichlet distribution
\begin{align}
    \pi_{A}^{(2)}&\sim \Dirichlet(\mathbf{1}_{N})
\end{align}
To perform adaptation on the transfer distribution, we also sample a smaller dataset of \emph{transfer} examples $D_{2} = \{(a_{i}, b_{i}\}_{i=1}^{m}$, with $m \ll n$ the size of the training set. In our experiment, we used $m=20$ transfer examples. We also used ancestral sampling on this new transfer distribution to acquire samples, similar to Equation~\eqref{eq:app-ancestral-sampling} (with $\pi_{A}^{(2)}$ instead of $\pi_{A}^{(1)}$).

Starting from the parameters estimated after the initial training on the training distribution, we perform a few steps of adaptation on the modules parameters $\theta_{A}$, $\theta_{B|A}$, $\theta_{B}$ \& $\theta_{A|B}$ using $T$ steps of gradient descent based on the transfer dataset $D_{2}$. The value of the likelihoods for both models is recorded as well, and computed as
\begin{align}
    \mathcal{L}_{A \rightarrow B} &= \prod_{t=1}^{T}P(\mathbf{a}_{t}\mid \theta_{A}^{(t)})P(\mathbf{b}_{t}\mid \mathbf{a}_{t}\,;\,\theta_{B|A}^{(t)})\nonumber\\
    \mathcal{L}_{B \rightarrow A} &= \prod_{t=1}^{T}P(\mathbf{b}_{t}\mid \theta_{B}^{(t)})P(\mathbf{a}_{t}\mid \mathbf{b}_{t}\,;\,\theta_{A|B}^{(t)}),
\end{align}
where $(\mathbf{a}_{t}, \mathbf{b}_{t})$ represents a mini-batch of examples from $D_{2}$, and the superscript $t$ on the parameters highlights the fact that these likelihoods are computed after $t$ steps of adaptation. This product over $t$ ensures that we monitor the progress of adaptation along the whole trajectory. In this experiment, we used $T=2$ steps of gradient descent on mini-batch of size 10 for the adaptation.

Finally, in order to update the structural parameter $\gamma$, we can use Proposition~\ref{prop:posterior} to compute the gradient of the loss $L$ with respect to $\gamma$:
\begin{align}
    \mathcal{R}(\gamma) &= -\log [\sigma(\gamma) \mathcal{L}_{A \rightarrow B} + (1 - \sigma(\gamma)) \mathcal{L}_{B \rightarrow A}]\\
    \frac{\partial \mathcal{R}}{\partial \gamma} &= \sigma(\gamma + \Delta) - \sigma(\gamma),
\end{align}
where $\Delta = \log \mathcal{L}_{A\rightarrow B} - \log \mathcal{L}_{B\rightarrow A}$. The update of $\gamma$ can be one step of gradient descent, or using any first-order optimizer like RMSprop. We perform multiple interventions over the course of meta-training by sampling multiple transfer distributions, and following the same steps of adaptation and update of the structural parameter $\gamma$.

In Figure~\ref{fig:two-variables-alphas-discrete}, we report the evolution of the structural parameter $\gamma$ (or rather, $\sigma(\gamma)$) as a function of the number of meta-training steps or, similarly, the number of different interventions made on the causal model. The model's belief $P(A\rightarrow B) = \sigma(\gamma)$ indeed converges to $1$, proving that the algorithm was capable of recovering the correct causal direction $A \rightarrow B$. 

\subsection{Discrete Variables with MLP Parametrization}
\label{sec:MLP2}

We consider a bivariate model similar to the ones defined above, where each random variable is sampled from a categorical distribution. Instead of expressing probabilities in a tabular form, we train $M=2$ simple feed-foward neural networks (MLP), one per conditional variable. MLP $i$ is the independent mechanism of causal variable $i$ that determines the conditional probability of the $N$ discrete choices for variable $i$, given its parents.

Each MLP receives $M$ concatenated $N$-dimensional one-hot vectors, masked appropriately according to the chosen causal structure $B$, i.e., with the $j$-th input
of the $i$-th MLP being multiplied by $B_{ij}$. Each directed edge presence or absence is thus indicated by $B_{ij}$, with $B_{ij}=1$ if variable $j$ is a direct causal parent of variable $i$. The MLP maps the $MN$ input units through one hidden layer that contains $H=4M$ hidden units and a ReLU non-linearity, and then maps the $H$ hidden units to $N$ output units and a softmax representing a predicted categorical distribution.

The causal structure belief is specified by an $M\times M$ matrix $\gamma$, with $\sigma(\gamma_{ij})$ the estimated probability that variable $i$ is directly caused by variable $j$. The causal structure $B_{ij}$ is drawn from $\mathrm{Ber}(\gamma_{ij})$, as per Algorithm~\ref{alg:main}. We generalize the estimator introduced for the 2-hypotheses case as per Appendix~\ref{sec:many-factors}, i.e., we use the gradient estimator in Equation~\ref{eq:gij}.

To evaluate the correctness of the structure being learnt, we measure the cross entropy between the ground-truth SCM and the learned SCM. In Figure \ref{fig:mlp_scm} we show this cross-entropy over different episodes of training for bivariate discrete distributions with either 10 categories or 100 categories. Both models are first pretrained for 100 examples with fully connected edges before starting training on the transfer distributions.

\subsection{Continuous Multimodal Variables}
\label{sec:continuous}

Consider a family of joint distributions $P_{\mu}(A, B)$ over the causal variables $A$ and $B$ sampled from the structural causal model (SCM): 
\begin{align} \label{eq:cont_scm}
    A &\sim P_{\mu}(A) = \mathcal{N}(\mu, \sigma^{2} = 4) \nonumber \\
    B &:= f(A) + N_{B}\qquad N_{B} \sim \mathcal{N}(\mu = 0, \sigma^{2} = 1)
\end{align}
where $f$ is a randomly generated spline and $N_{B}$ is sampled i.i.d from the unit-normal distribution.

To obtain the spline, we sample the $K$ points $\{x_{k}\}_{k = 1}^{K}$ uniformly spaced from the interval $[-R_A, R_A]$, and another $K$ points $\{y_{k}\}_{k = 1}^{K}$ uniform randomly from the interval $[-R_B, R_B]$. This yields $K$ pairs $\{(x_{k}, y_{k})\}_{k = 1}^{K}$, which make the knots of a second-order spline. We set $K = 8$, $R_A = R_B = 8$ for our experiments. In Figure~\ref{fig:encoder_scm}, we plot samples from one such SCM for the training distribution ($\mu = 0$) and two transfer distributions ($\mu = \pm 4$). 

\begin{figure}[ht]
    \centering
    \ificml
\else
    \includegraphics[width=0.6\linewidth]{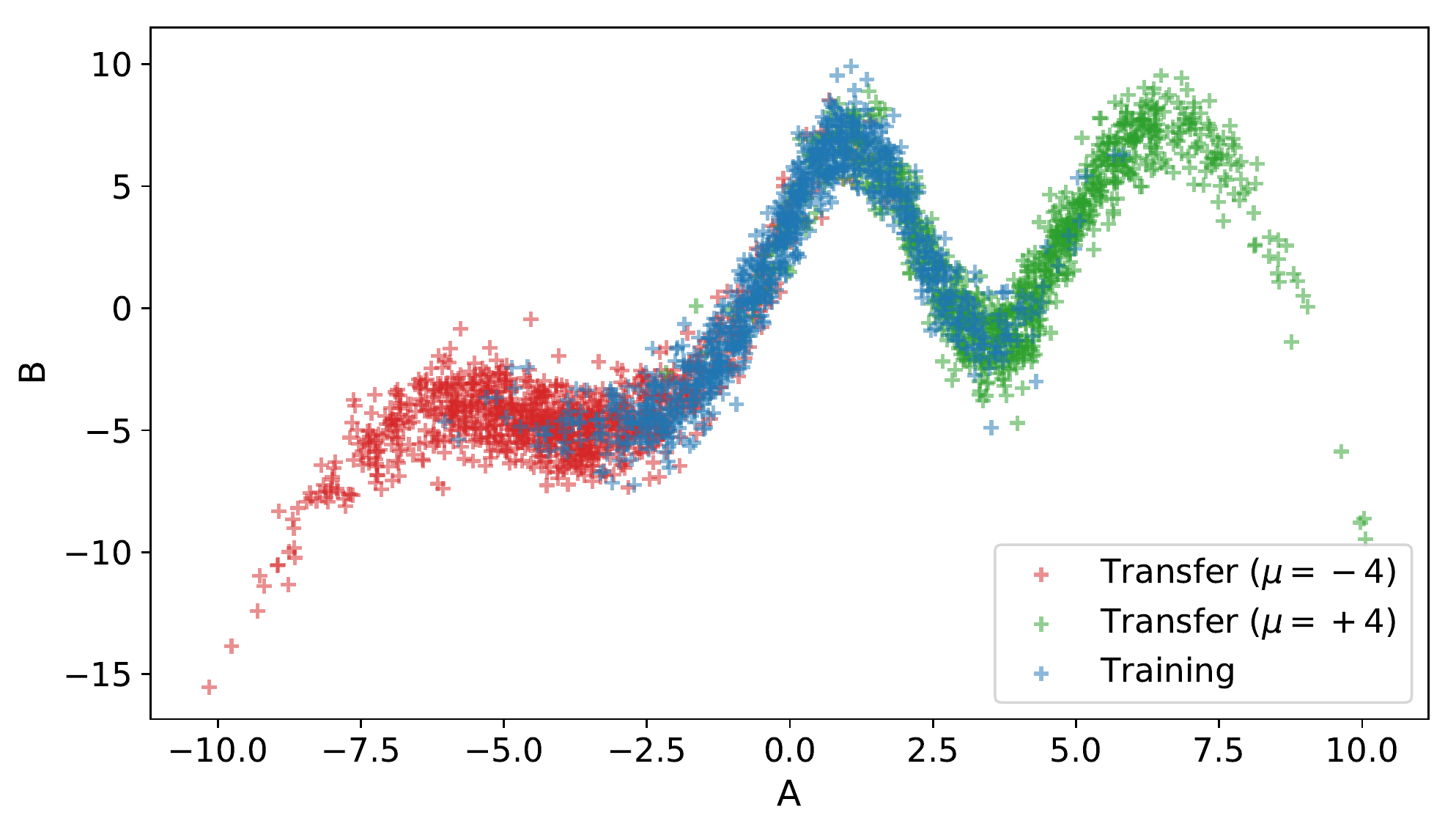}
    \fi
    \caption{Train (red) and transfer (green and blue) samples from an SCM generated with the procedure described in Equation~\eqref{eq:cont_scm}. The green data-points are sampled from $P_{\mu=(-4)}(A, B)$, whereas the blue data-points are samples from $P_{\mu=(+4)}(A, B)$ and the red data points (training set) are from $P_{\mu=0}(A, B)$.}
    \label{fig:encoder_scm}
\end{figure}

The conditionals $P(B \mid A\,;\, \theta_{B|A})$ and $P(A \mid B\,;\, \theta_{A|B})$ are parameterized as 2-layer Mixture Density Networks \cite{bishop1994mixture} with $32$ hidden units and $10$ components. The marginals $P(A \mid \theta_{A})$ and $P(B \mid \theta_{B})$ are parameterized as Gaussian Mixture Models, also with $10$ components. The training now follows as described below. 

Similar to Appendix~\ref{sec:bivariate-experiment-discrete}, we first pre-train the modules corresponding to the conditionals and marginals on the training distribution. To that end, we select $P_{\mu=0}(A, B)$ as the training distribution, sample a (large) training dataset $\{(a_{i}, b_{i})\}_{i=1}^{n}$ from it using ancestral sampling, and solve the following two problems independently until convergence: 
\begin{align}
\max_{\theta_{A}, \theta_{B|A}} &\sum_{i=1}^{n}\log P(a_{i} \mid \theta_{A})P(b_{i} \mid a_{i}; \theta_{B|A}) \\
\max_{\theta_{B}, \theta_{A|B}} &\sum_{i=1}^{n}\log P(b_{i} \mid \theta_{B})P(a_{i} \mid b_{i}; \theta_{A|B})
\end{align}

The adaptation performance of $A \to B$ and $B \to A$ models can now be evaluated on transfer distributions. For a $\mu$ sampled uniformly in $[-4, 4]$, we select $P_{\mu}(A', B')$ as the transfer distribution, and denote with $(A', B')$ samples from it. Both models are fine-tuned on $P_{\mu}(A', B')$ for $T = 10$ iterations (see Algorithm~\ref{alg:main}), and the area under the corresponding negative-log-likelihood curves becomes the regret: 
\begin{align} \label{eq:cont_regret_A2B}
\mathcal{R}_{A \to B} = -\sum_{t = 1}^{T} \log P(B' | A'; \theta_{A \to B}^{(t)}) P(A' | \theta_{A \to B}^{(T)})
\end{align}
and likewise for $\mathcal{R}_{B \to A}$. In these experiments, the modules corresponding to the marginals (ie. GMM) are learned \emph{offline} via Expectation Maximization, and we denote with $P(A' | \theta_{A \to B}^{(T)})$ the trained model. These can now be used to define the following meta-objective for the structural meta-parameter $\gamma$: 
\begin{align} \label{eq:gamma_regret}
\mathcal{R}(\gamma) = \log[\sigma(\gamma) e^{\mathcal{R}_{A \to B}} + (1 - \sigma(\gamma)) e^{\mathcal{R}_{B \to A}}]
\end{align}

The structural regret $\mathcal{R}(\gamma)$ is now minimized with respect to $\gamma$ for 200 iterations (updates of $\gamma$). 
Figure~\ref{fig:encoder_gamma_evo} shows the evolution of $\sigma(\gamma)$ as training progresses. This is expected, given that we expect the causal model to perform better on the transfer distributions, i.e. we expect $\mathcal{R}_{A \to B} < \mathcal{R}_{B \to A}$ in expection. Consequently, assigning a larger weight to $\mathcal{R}_{A \to B}$ optimizes the objective.

\begin{figure}[ht]
    \centering
    \ificml
\else
    \includegraphics[width=0.6\linewidth]{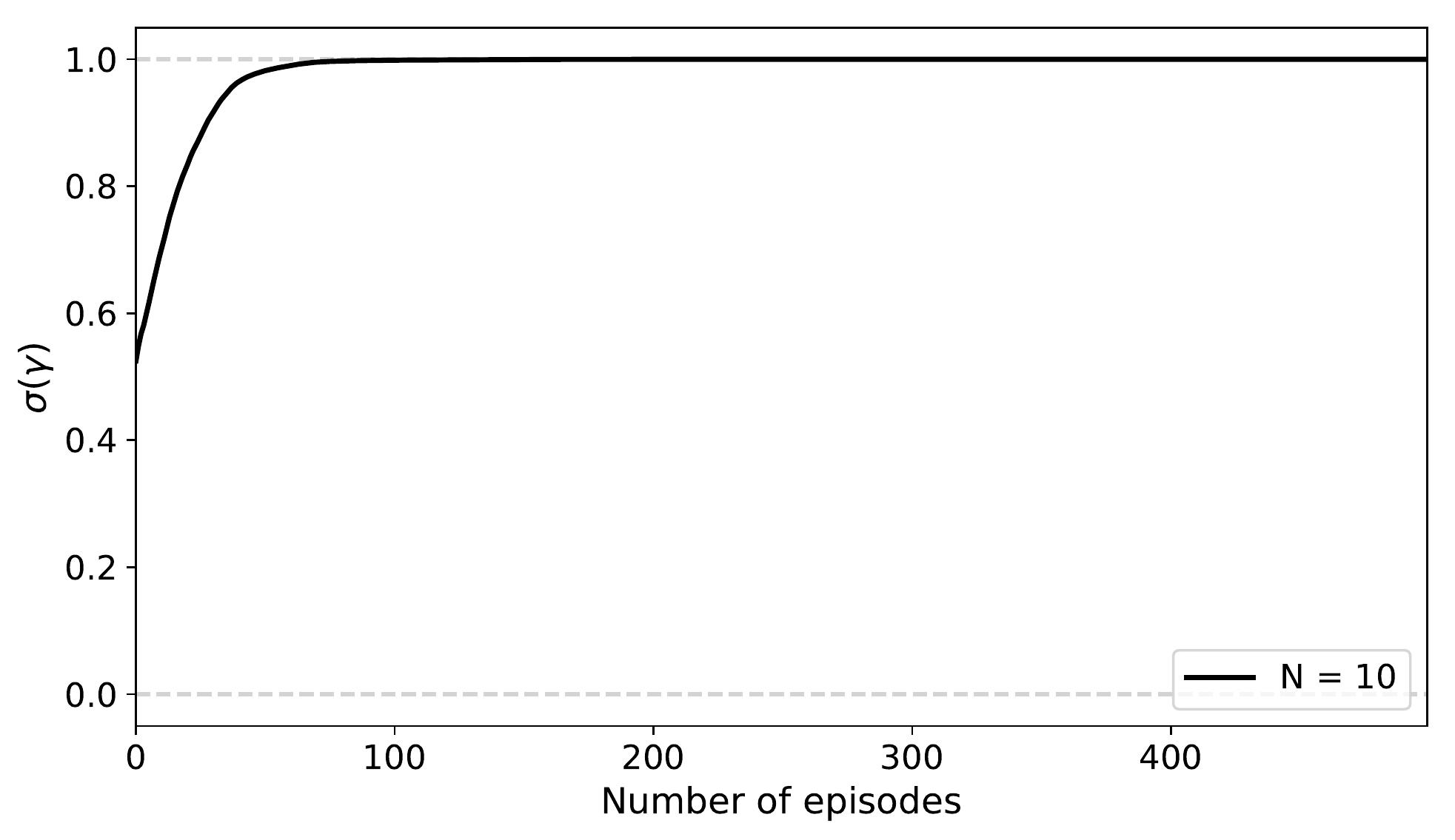}
    \fi
    \caption{Evolution of the sigmoid of structural meta-parameter $\sigma(\gamma)$ with training iterations. It is indeed expected to increase if $A \to B$ is the true causal graph (see Equation~\eqref{eq:gamma_regret}).}
    \label{fig:encoder_gamma_evo}
\end{figure}

\subsection{Linear Gaussian Model}
\label{sec:linear-gaussian}
In this experiment, the two variables we consider are vectors (i.e. $A \in \mathbb{R}^d$ and
$B \in \mathbb{R}^d$). The ground truth causal model is given by
\begin{align} \label{eq:lingauss_scm}
    A &\sim \mathcal{N}(\mu_A, \Sigma_A) \nonumber \\
    B &:= \beta_1 A + \beta_0 + N_{B}\qquad N_{B} \sim \mathcal{N}(0, \Sigma_{B})
\end{align}
where $\mu_A \in \mathbb{R}^{d}$, $\beta_0 \in \mathbb{R}^{d}$ and $\beta_1 \in \mathbb{R}^{d \times d}$. $\Sigma_A$ and $\Sigma_B$ are $d \times d$ covariance matrices\footnote{Ground truth parameters $\mu_A$, $\beta_1$ and $\beta_0$ are sampled from a Gaussian distribution, while $\Sigma_A$ and $\Sigma_B$ are sampled from an inverse Wishart distribution.}. In our experiments, $d=100$. Once again, we want to identify the correct causal direction between $A$ and $B$. To do so, we consider two models: $A \rightarrow B$ and $B \rightarrow A$. We parameterize both models symmetrically:

\begin{align} \label{eq:lingauss_models}
    P_{A \rightarrow B}(A) &= \mathcal{N}(A; \hat{\mu}_A, \hat{\Sigma}_A) \nonumber \\
    P_{A \rightarrow B}(B \mid A=a) &= \mathcal{N}(B; \hat{W}_1 a + \hat{W}_0, \hat{\Sigma}_{A \rightarrow B}) \nonumber \\
    P_{B \rightarrow A}(B) &= \mathcal{N}(B; \hat{\mu}_B, \hat{\Sigma}_B) \nonumber \\
    P_{B \rightarrow A}(A \mid B=b) &= \mathcal{N}(B; \hat{V}_1 b + \hat{V}_0, \hat{\Sigma}_{B \rightarrow A})
\end{align}
Note that each covariance matrix is parameterized using the Cholesky decomposition. Unlike previous experiments, we are not conducting any pre-training on actual data. Instead, we fix the parameters of both models to their exact values according to the ground truth parameters introduced in Equation \ref{eq:lingauss_scm}. For model $A \rightarrow B$, this can be done trivially. For the second model, we can compute its exact parameters analytically. Once the exact parameters are set, both models are equivalent in the sense that $P_{A \rightarrow B}(A, B) = P_{B \rightarrow A}(A, B)$ $\forall A, B$.

Each meta-learning episode starts by initializing the parameters of both models to the values identified during the pre-training. Afterward, a transfer distribution is sampled (i.e. $\mu_A \sim \mathcal{N}(0, I)$). Then, both models are trained on samples from this distribution, for 10 iterations only. During this adaptation, the log-likelihoods of both models are accumulated in order to compute $\mathcal{L}_{A \rightarrow B}$ and $\mathcal{L}_{B \rightarrow A}$. At this stage, we compute the meta objective estimate $\regret=-\log \left[ \sigm(\gamma) \mathcal{L}_{A \rightarrow B} + (1-\sigm(\gamma)) \mathcal{L}_{B \rightarrow A} \right]$, compute its gradient w.r.t. $\gamma$ and update $\gamma$.

Figure \ref{fig:LinGauss_gamma_evo} shows that, after 200 episodes, $\sigma(\gamma)$ converges to 1, indicating the success of the method on this particular task. 

\begin{figure}[ht]
    \centering
    \ificml
\else
    \includegraphics[width=0.6\linewidth]{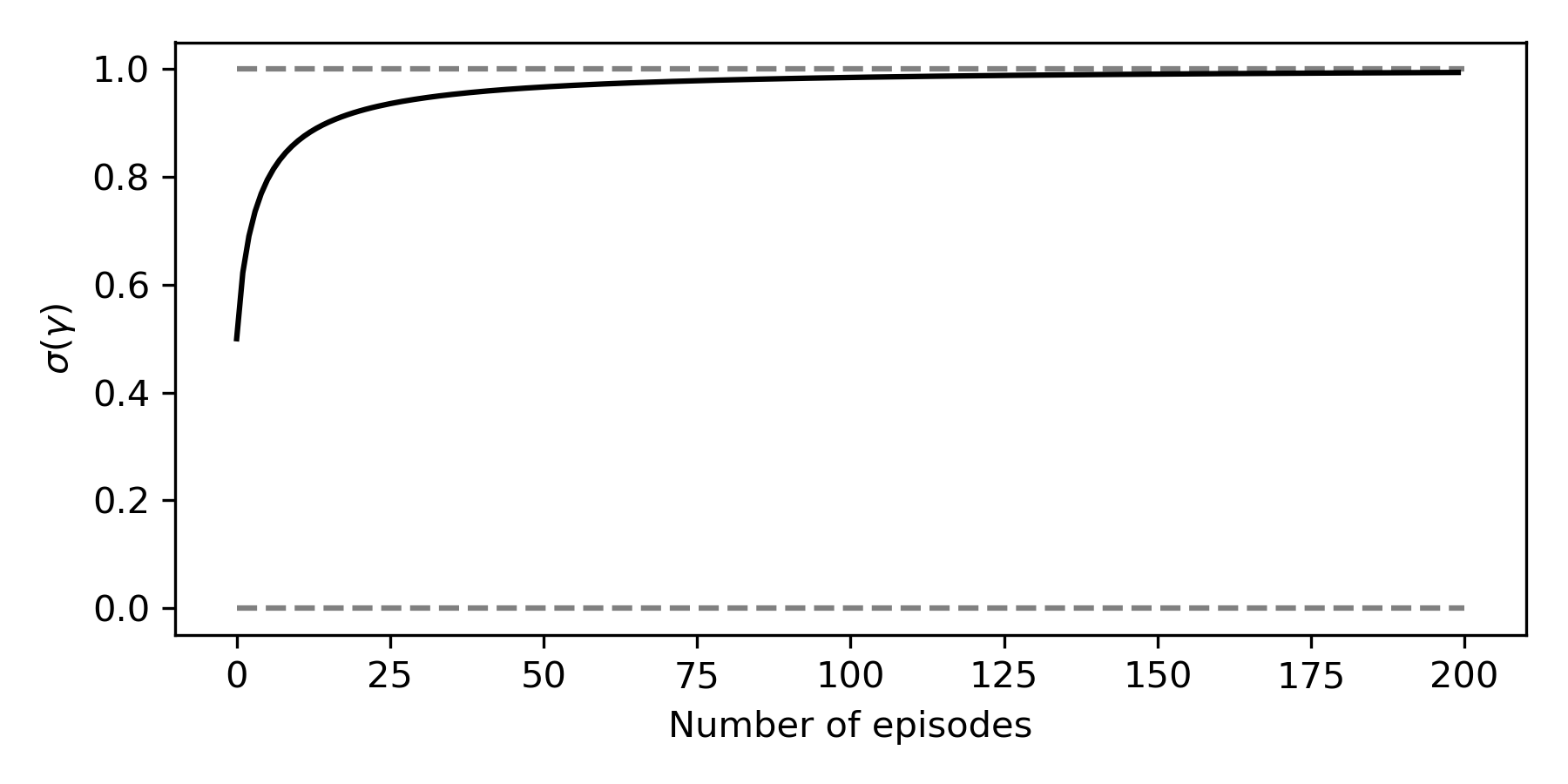}
    \fi
   \caption{Convergence of the causal belief (to the correct answer) as a function of the number of meta-learning episodes, for the linear Gaussian experiments.}
    \label{fig:LinGauss_gamma_evo}
\end{figure}

\section{Results on Learning the Correct Encoder}
\label{sec:encoder-experiment}

The causal variables $(A, B)$ are sampled from the distribution described in Eqn~\ref{eq:cont_scm}, and are mapped to observations $(X, Y) \sim P_{\mu}(X, Y)$ via a hidden (and a priori unknown) decoder $\mathcal{D} = R(\theta_{\mathcal{D}})$, where $R$ is a rotation matrix. The observations are then mapped to the hidden state \mbox{$(U, V) \sim P_{\mu}(U, V)$} via the encoder $\mathcal{E} = R(\theta_{\mathcal{E}})$. The computational graph is depicted in Figure~\ref{fig:encoder_arch}. 

Analogous to Equation~\ref{eq:gamma_regret} in Appendix~\ref{sec:continuous}, we now define the regret over the variables $(U, V)$ instead of $(A, B)$:
\begin{align}
\mathcal{R}(\gamma, \theta_{\mathcal E}) = \log[\sigma(\gamma) e^{\mathcal{R}_{U \to V}} + (1 - \sigma(\gamma)) e^{\mathcal{R}_{V \to U}}]
\end{align}
where the dependence on $\theta_{\mathcal{E}}$ is implicit in $(U, V)$. In every meta-training iteration, the $U \to V$ and $V \to U$ models are trained on the training distribution $P_{\mu=0}(U, V)$ for $T' = 20$ iterations. Subsequently, the regrets $\mathcal{R}_{U \to V}$ and $\mathcal{R}_{V \to U}$ are obtained by a process identical to that described in Equation~\ref{eq:cont_regret_A2B} of Appendix~\ref{sec:continuous} (albeit with variables $(U, V)$ and $T = 5$). Finally, the gradients of $\mathcal R(\gamma, \theta_{\mathcal{E}})$ are evaluated and the meta-parameters $\gamma$ and $\theta_{\mathcal{E}}$ are updated. This process is repeated for $1000$ meta-iterations, and Figure~\ref{fig:encoder_evo} shows the evolution of $\theta_{\mathcal{E}}$ as training progresses (where $\theta_{\mathcal{D}}$ has been set to $-\frac{\pi}{4}$). Further, Figure~\ref{fig:enc_gamma_evo} shows the corresponding evolution of the structural parameter $\gamma$. 

\begin{figure}[ht]
    \centering
    \ificml
\else
    \includegraphics[width=0.6\linewidth]{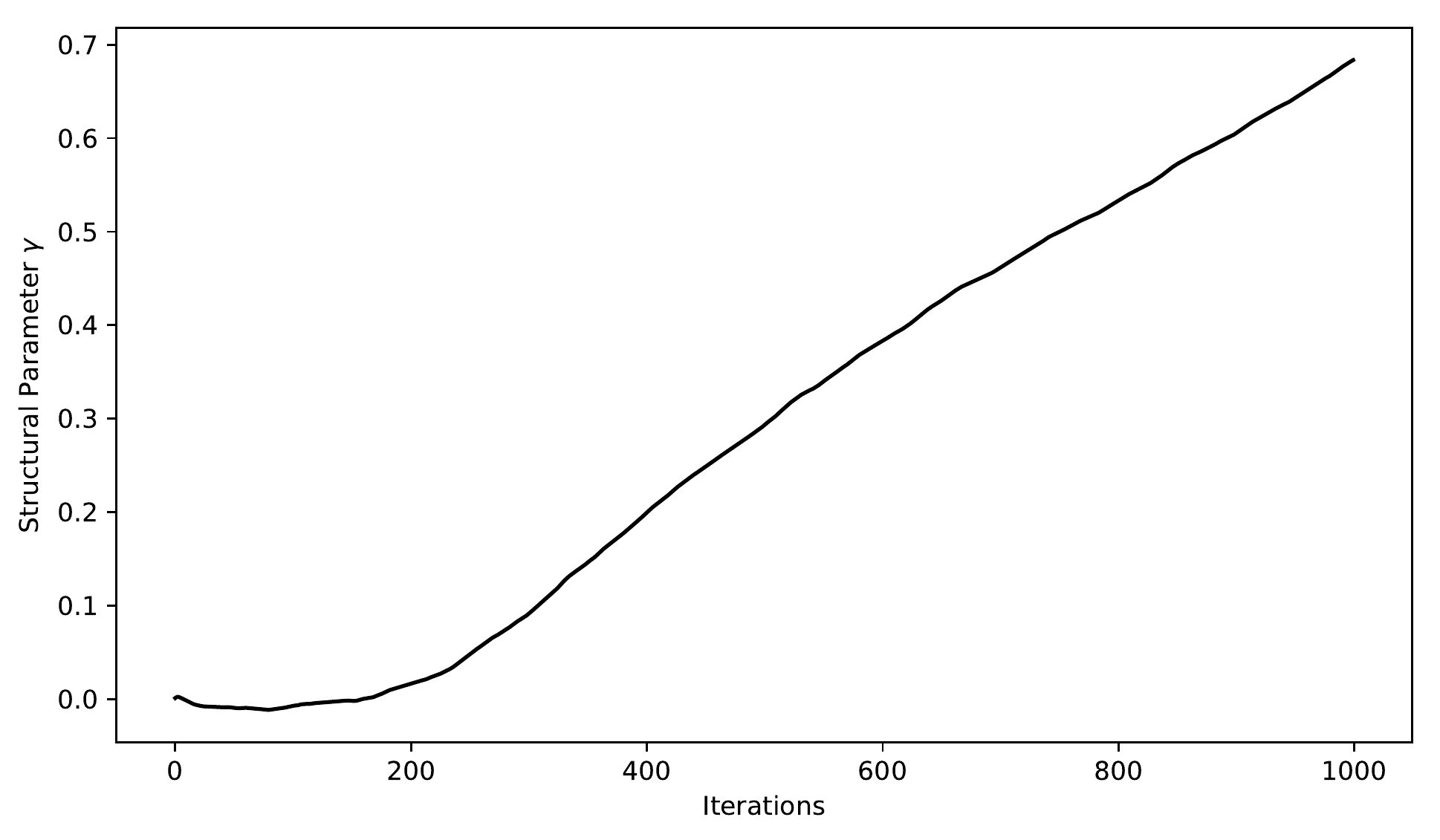}
    \fi
    \caption{Evolution of the structural parameter $\gamma$ as training progresses with the encoder. The corresponding evolution of the encoder parameter $\theta_{\mathcal{E}}$ is shown in Figure~\ref{fig:encoder_evo}. Observe that the system converges to $\theta = 0$, implying that the correct causal direction is $U \to V$ and the parameter $\gamma$ should increase with meta-training iterations.}
    \label{fig:enc_gamma_evo}
\end{figure}

\end{document}